\documentclass{article}

 \usepackage[main, final]{neurips_2026}

\usepackage[utf8]{inputenc} 
\usepackage[T1]{fontenc}    
\usepackage[colorlinks=true,citecolor=brown,urlcolor=gray,linkcolor=BlueViolet]{hyperref}       
\usepackage{url}            
\usepackage{booktabs}       
\usepackage{amsfonts}       
\usepackage{nicefrac}       
\usepackage{microtype}      
\usepackage[table,dvipsnames]{xcolor}
\usepackage{algorithm}
\usepackage{algpseudocode}
\usepackage{graphicx}
\usepackage{enumitem}
\usepackage{multirow}
\usepackage{makecell}
\usepackage{pifont}
\usepackage{wrapfig}
\usepackage{subcaption}
\usepackage{fontawesome5}
\usepackage{needspace}

\usepackage{amsmath,amsfonts,bm,eqnarray}
\usepackage{cancel}
\usepackage{amsthm}
\usepackage{tikz}
\usetikzlibrary{tikzmark}

\newtheorem{proposition}{Proposition}[section]

\newtheorem{innercustomgeneric}{\customgenericname}
\providecommand{\customgenericname}{}
\newcommand{\newcustomtheorem}[2]{%
  \newenvironment{#1}[1]
  {%
   \renewcommand\customgenericname{#2}%
   \renewcommand\theinnercustomgeneric{##1}%
   \innercustomgeneric
  }
  {\endinnercustomgeneric}
}

\newcustomtheorem{customThm}{Theorem}
\newcustomtheorem{customLemma}{Lemma}
\newcustomtheorem{customCor}{Corollary}
\newcustomtheorem{customProposition}{Proposition}

\def\Propref#1{Proposition~\ref{#1}}

\def\Tabref#1{Table~\ref{#1}}

\def\Figref#1{Fig.~\ref{#1}}

\def\appref#1{Appendix~\ref{#1}}
\def\Secref#1{Sec.~\ref{#1}}

\def\eqref#1{equation~\ref{#1}}
\def\Eqref#1{Eqn.~\ref{#1}}

\def\1{\bm{1}}

\def\vh{{\bm{h}}}

\def\vx{{\bm{x}}}
\def\vy{{\bm{y}}}
\def\vz{{\bm{z}}}

\DeclareMathAlphabet{\mathsfit}{\encodingdefault}{\sfdefault}{m}{sl}
\SetMathAlphabet{\mathsfit}{bold}{\encodingdefault}{\sfdefault}{bx}{n}

\renewcommand{\tilde}{\widetilde}
\renewcommand{\hat}{\widehat}
\renewcommand{\frac}{\tfrac}

\renewcommand{\cite}{\citep}

\usepackage{xspace}

\newcommand{\ours}{SAUCE\xspace}

\definecolor{checkgreen}{RGB}{76,175,80} 
\definecolor{spigreen}{HTML}{3F8E5C} 
\definecolor{oursrow}{gray}{0.93}
\title{Sequential Probabilistic Uncertainty Estimation for Parallel Multi-Agent Reasoning Systems}

\author{%
  Tunyu Zhang$^{1,\dagger}$
  \And
  Zihao Zhao$^{2}$
  \And
  Yusong Zhao$^{2}$
  \And
  Haizhou Shi$^{1}$
  \And
  Zhuohang Li$^{3}$
  \AND
  Haoxian Chen$^{4}$
  \And
  Hao Wang$^{2,\dagger,\ddagger}$
  \And
  Dimitris N. Metaxas$^{1,\dagger, \ddagger}$
}

\begin{document}

\maketitle

\begingroup
\makeatletter
\long\def\@makefntext#1{%
  \setlength{\parindent}{0pt}%
  \noindent#1%
}
\makeatother
\renewcommand{\thefootnote}{}

\footnotetext[0]{%
  $^{1}$Rutgers University.
  $^{2}$University of Illinois Urbana-Champaign.
  $^{3}$Vanderbilt University.
  $^{4}$Columbia University.
  \par
  $^{\dagger}$Correspondence to:
  Tunyu Zhang (\nolinkurl{ty.zhang@rutgers.edu}), 
  Hao Wang (\nolinkurl{hwangml@illinois.edu}), 
  Dimitris N. Metaxas (\nolinkurl{dmn@cs.rutgers.edu}).
  $^{\ddagger}$Equal advising.
}
\endgroup

\begin{abstract}
LLM-based multi-agent systems (MAS) 
have attracted growing attention for improving reasoning through interaction among multiple agents. 
In this work, we focus on \emph{parallel multi-agent reasoning systems}, where several agents solve the same problem over multiple rounds and aggregate their outputs into a final answer. 
Despite their strong reasoning performance, uncertainty estimation for such systems remains underexplored: the reliability of a MAS depends not only on individual generations, but also on how agents interact and evolve across rounds. 
We propose \textbf{\ours} (Sequential Agent Uncertainty through Consensus Evolution), a lightweight, training-free uncertainty estimator that formulates MAS uncertainty as sequential inference over a latent system-level belief. 
\ours aggregates round-level agreement and generation-uncertainty signals through a filtering-style update. 
Across five backbones, five benchmarks, and two MAS protocols, \ours improves misclassification detection, selective prediction, and calibration over a broad set of uncertainty estimation baselines, including standard log-likelihood-based methods and MAS-specific estimators.

\smallskip
\noindent\textbf{\faGithub\ Code:}
\href{https://github.com/Tyrion58/SAUCE}
{\texttt{github.com/Tyrion58/SAUCE}}
\par\noindent\textbf{\raisebox{-0.15em}{\includegraphics[height=1em]{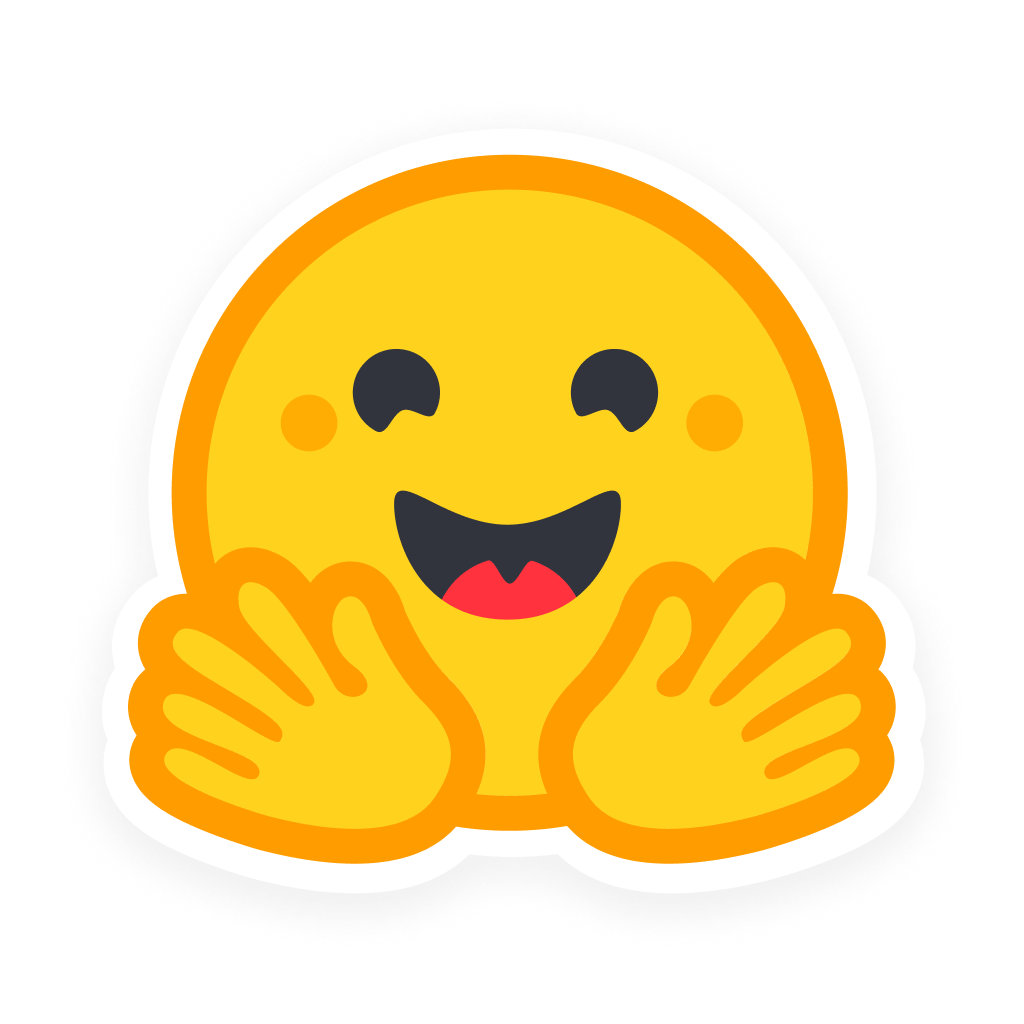}}\ Data:}
\href{https://huggingface.co/datasets/Tyrion279/SAUCE-Qwen3-4B-Rollouts}
{\texttt{Tyrion279/SAUCE-Qwen3-4B-Rollouts}}
\end{abstract}

\vspace{-4mm}
\section{Introduction}

Multi-agent systems (MAS) built with large language models (LLMs) have demonstrated strong capabilities in complex reasoning tasks~\cite{li2024survey, talebirad2023multi, he2025llm, yu2024fincon,MASProVe}. 
A common design is \emph{parallel multi-agent reasoning}~\cite{cui2025free, du2024improving, humulti, liu2023dynamic}, such as debate-style systems~\cite{chan2023chateval, du2024improving}, where multiple agents solve the same problem and exchange responses over several rounds. In these systems, agents act in parallel within each round, while interactions evolve sequentially across rounds. 
As such systems are explored for complex reasoning, planning, and decision-making tasks~\cite{xu2025nuclear, coeckelbergh2025llms, guo2024large, lee2025reliable}, this motivates the need for reliable system-level uncertainty estimation for multi-agent interactions.

Uncertainty estimation has been extensively studied for single LLMs, including sampling-based~\cite{farquhar2024detecting, manakul2023selfcheckgpt, kossen2024semantic, malinin2020uncertainty}, Bayesian~\cite{yangbayesian, wang2024blob, shitraining}, and internal signal-based approaches~\cite{xia2025survey, kossen2024semantic, ren2023self, fu2025deep}. 
Extending these methods to MAS, however, is non-trivial. 
Sampling-based and Bayesian approaches become costly when repeated across agents and rounds. 
Signal-based methods are more computationally practical because they rely on quantities already available during interaction, but they introduce a different challenge: these signals are heterogeneous across agents and fluctuate
over rounds. 
For parallel reasoning MAS, uncertainty estimation must aggregate heterogeneous agent-level signals into a coherent system-level estimate while tracking how these signals change across rounds.

Crucially, multi-agent interaction is a \emph{sequential} process~\cite{liang2024encouraging, du2024improving}. 
As agents exchange information over rounds, reliability may change, so uncertainty should be updated progressively rather than inferred from a single-step signal or a static summary (\Figref{fig:main}~(A)). 
Thus, MAS uncertainty is not merely a property of the final output, but of the \emph{interaction process} itself. 
This raises a key question:
\begin{center}
\emph{Do the interaction dynamics of a MAS inform its internal uncertainty?}
\end{center}

We address this question by modeling MAS uncertainty as a \textbf{sequential inference problem} over interaction rounds.
We propose \textbf{\ours} (\textbf{S}equential \textbf{A}gent \textbf{U}ncertainty through \textbf{C}onsensus \textbf{E}volution), a lightweight, training-free uncertainty estimator for parallel reasoning MAS.
As shown in~\Figref{fig:main}~(B), 
inspired by hierarchical Bayesian deep learning~\cite{BDL,BDLSurvey}, 
\ours treats each round as noisy evidence about an evolving system-level belief, rather than as an independent snapshot.
Round-level agreement provides the evidence, while generation uncertainty determines how strongly the evidence should update the belief.
This yields a filtering-style update that propagates the previous belief state and revises uncertainty according to the current round's generation uncertainty.
As a result, uncertainty is inferred from the full interaction trajectory rather than from a single final answer.
\ours uses only signals extracted from the MAS inference trace, requiring no additional sampling or access to model weights or hidden states.
Empirically, \ours outperforms log-likelihood baselines and MAS-specific uncertainty estimators across Debate~\cite{du2024improving} and DyLAN~\cite{liu2023dynamic} protocols.

\textbf{Contributions.}\quad 
Our contributions are summarized as follows:
\begin{enumerate}[nosep,leftmargin=12pt]
    \item \textbf{A system-dynamics perspective for MAS uncertainty estimation.}
    We argue that uncertainty in parallel reasoning MAS should be estimated from
    how consensus and reliability evolve across rounds, rather than from a static
    final-output signal.

    \item \textbf{A sequential inference framework for consensus aggregation.}
    We formulate MAS uncertainty estimation as a filtering-style update over a
    latent system-level belief, where round-level evidence is weighted by its
    estimated generation uncertainty.

    \item \textbf{A lightweight estimator for practical MAS.}
    We instantiate \ours with majority-vote consensus and token-level generation
    entropy, yielding a training-free method that requires minimal model access
    and applies to both open- and closed-source models with top-$k$ log-probabilities.
\end{enumerate}

\begin{figure}
    \centering
    \includegraphics[width=\linewidth]{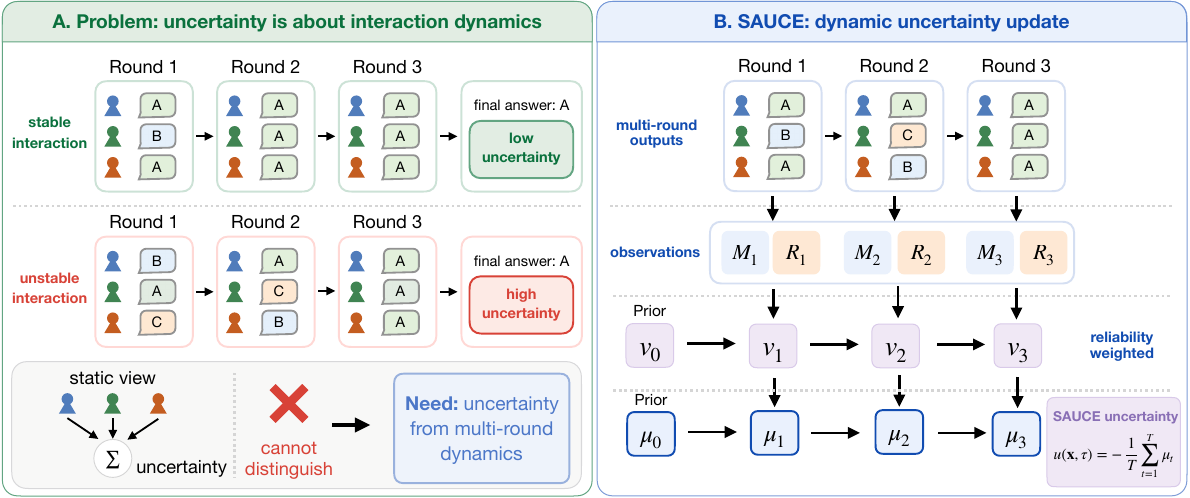}
    \caption{Overview. \textbf{(A)} Two MAS trajectories can reach the same final answer but differ greatly in reliability: a \emph{stable} run whose outputs converge across rounds (top) is more trustworthy than an \emph{unstable} run whose outputs oscillate (bottom). A static view that only considers the final output cannot distinguish them, motivating uncertainty estimation from interaction dynamics. \textbf{(B)} \ours models multi-round interaction as a sequence of round-level observations $(M_t, R_t)$, where agreement is weighted by generation uncertainty, and performs a filtering-style update over a latent system-level belief. The resulting trajectory summary defines the \ours uncertainty score $u(\vx,\tau)$ (\Eqref{eq:final-u}).}
    \label{fig:main}
\vspace{-3mm}
\end{figure}

\section{Preliminaries}
In this section, we formalize the multi-agent reasoning process and define the problem of uncertainty estimation for MAS.

\subsection{Multi-Agent Reasoning Trajectory}

Given an input problem $\vx$, a MAS consists of $N$ agents $\mathcal{A}=\{A_1,\ldots,A_N\}$ that solve the problem through $T$ rounds of interaction.
At round $t$, agent $A_i$ produces an output
\begin{align}
\vz_t^i \sim \pi_i(\cdot \mid \vx, \vh_t^i),
\end{align}
where $\vh_t^i$ denotes the information available to agent $A_i$ before round $t$, including its own previous outputs and the messages observed from other agents.
The full multi-agent interaction trajectory is
\begin{align}
\tau = \left(\vz_1^{1:N}, \vz_2^{1:N}, \ldots, \vz_T^{1:N}\right),
\end{align}
which is sampled from the MAS rollout distribution induced by the agents and the interaction protocol.
The final prediction is obtained by applying an aggregation function $g$ to the input and the trajectory:
\begin{align}
\hat{\vy} = g(\vx,\tau).
\end{align}
\textbf{Scope: Parallel Reasoning MAS.}\quad
This paper focuses on \emph{parallel} multi-agent reasoning~\cite{du2024improving, liang2024encouraging, liu2023dynamic}, where all $N$ agents generate responses in parallel at each round, conditioned on the same input $\vx$ and the prior interaction history.
In the protocols considered in this paper, the final prediction is typically obtained by aggregating the round-$T$ parsed answers, e.g., by majority vote~\cite{wang2022self, kaesberg2025voting, choidebate}.
Debate~\cite{du2024improving} and DyLAN~\cite{liu2023dynamic} are representative protocols of this form.

\subsection{Sequential Uncertainty Estimation for MAS}

Given an input $\vx$, a reasoning trajectory $\tau$, and the final prediction $\hat{\vy}=g(\vx,\tau)$, our goal is to estimate a trajectory-level uncertainty score
\begin{align}
u(\vx,\tau) \in \mathbb{R},
\end{align}
where a larger value indicates lower confidence in $\hat{\vy}$.

\textbf{Dynamics in MAS.}\quad
Unlike static confidence measures that summarize only the final output or aggregate signals without temporal structure, we model MAS uncertainty as a quantity induced by the interaction process.
At each round, agent outputs induce round-level evidence about the system state, and the uncertainty is updated sequentially as new evidence arrives:
\begin{align}
u_t = \Phi(u_{t-1}, e_t), 
\qquad 
u(\vx,\tau)=\mathcal{A}(u_{1:T}),
\end{align}
where $e_t$ denotes the evidence extracted from round $t$, $u_t$ denotes the system uncertainty state after round $t$, $\Phi$ is a sequential update rule, and $\mathcal{A}$ maps the uncertainty trajectory to the final score.

\textbf{Challenges.}\quad
This dynamic view is essential because the final prediction arises from multi-round communication, where agents may reinforce, revise, or contradict one another.
A static summary, such as the log-likelihood of the final answer or an order-invariant average of per-round signals, is blind to the \emph{temporal order} of evidence: two trajectories with identical average signals but opposite dynamics may receive the same score despite carrying very different reliability.
MAS uncertainty should therefore be updated progressively as new evidence emerges across rounds, rather than inferred from a single-step signal or a static aggregation.

Motivated by this view, we formulate MAS uncertainty estimation as a sequential inference problem.
\ours instantiates this principle by integrating round-level signals into the uncertainty state $u_t$, leading to the method in~\Secref{sec:method} and the probabilistic interpretation in~\Secref{sec:theory}.

\section{Method}
\label{sec:method}
We present \ours, a lightweight uncertainty estimator for parallel reasoning MAS.
\ours extracts round-level evidence from generations of agents, and updates uncertainty across rounds through a filtering-style recursion.
This recursion can be interpreted as \textbf{posterior uncertainty inference over a latent system-level belief}, with the full probabilistic view given in~\Secref{sec:theory}.
\subsection{Round-Level Signals from Agent Outputs}
\label{sec:reliability}

At each interaction round $t$, the $N$ agents generate outputs in parallel,
$\{\vz_t^i\}_{i=1}^N$.
This parallel structure naturally exposes two round-level signals at no extra cost:
a cross-agent \emph{agreement} signal and a per-agent \emph{generation-uncertainty} signal.

\textbf{Agreement Signal $M_t$.}\quad
For each output $\vz_t^i$, we extract the parsed answer $a_t^i$.
Let $m_t$ denote the modal answer at round $t$.
We define $M_t$ as the \emph{modal-answer share}, i.e., the fraction of agents whose answer matches $m_t$:
\begin{align}\label{eq:obs}
M_t \;=\; \frac{1}{N}\sum\nolimits_{i=1}^{N} \mathbb{1}[\,a_t^i = m_t\,].
\end{align}
A higher $M_t$ reflects stronger cross-agent consensus. 

\textbf{Generation-Uncertainty Signal $R_t$.}\quad
Cross-agent agreement should be trusted more when the underlying generations are confident.
We measure round-level generation uncertainty by averaging token-level predictive entropy across agents:
\begin{align}\label{eq:R_round}
R_t \;=\; \frac{1}{N}\sum\nolimits_{i=1}^{N} H_{\mathrm{tok}}(\vz_t^i),
\end{align}
where $H_{\mathrm{tok}}(\vz_t^i)$ averages predictive entropy~\cite{malinin2020uncertainty} over the tokens of $\vz_t^i$, computed from the top-$k$ log-probabilities.
Larger $R_t$ indicates higher generation uncertainty, so the round-level agreement should be trusted less.
Together, $(M_t, R_t)$ captures two sources of MAS uncertainty: cross-agent disagreement, reflected by low $M_t$, and generation uncertainty, reflected by high $R_t$.


\subsection{Sequential Uncertainty Update}\label{sec:sequential}

\ours maintains a round-level confidence estimate $\mu_t$ throughout the multi-agent interaction, together with a stability weight $v_t$ that controls how strongly new evidence updates this estimate.
Both quantities are initialized by hyperparameters, $\mu_0$ and $v_0$.
The final uncertainty score is obtained by negating the temporal average of $\mu_t$.
This sequential estimation process alternates between two steps: confidence \emph{propagation} and reliability-weighted confidence \emph{update}.

\begin{wrapfigure}{r}{0.40\linewidth}
    \centering
    \includegraphics[width=\linewidth]{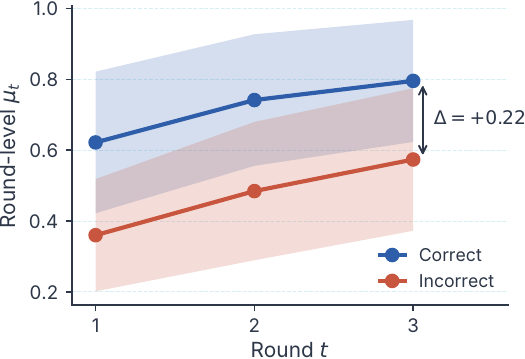}
    \caption{Round-level $\mu_t$ on Llama-3.1-8B / MATH-500. Both classes grow more confident with each round, but correct trajectories converge to a consistently higher consensus. Shading shows $\pm 1$ standard deviation across trajectories.}
    \label{fig:mu-per-round}
    \vspace{1.0em}
\end{wrapfigure}

\textbf{Propagation.}\quad
At round $t$, \ours first propagates the previous state:
\begin{align}\label{eq:trans-mv}
\tilde{\mu}_t = \lambda \mu_{t-1}, \qquad
\tilde{v}_t = \lambda^2 v_{t-1} + q, \qquad t\in [1,T],
\end{align}
where $\lambda \in (0,1)$ controls temporal coupling across rounds, and $q$ captures additional uncertainty introduced during interaction. The two limits motivate this open interval: as $\lambda \to 0$ the state resets each round (no temporal memory), and as $\lambda \to 1$ the previous belief is retained without decay; an interior $\lambda$ balances these two extremes.

\ours then combines the propagated state with the current round signals through a reliability-dependent weight:
\begin{align}\label{eq:alpha}
\alpha_t \;=\; \tilde{v}_t/(\tilde{v}_t + R_t), \quad t\in [1,T].
\end{align}

\textbf{Update.}\quad
The weight $\alpha_t$ determines how much \ours trusts the current round. When the round is reliable, $R_t$ is small and $\alpha_t$ is large, so the estimate moves quickly toward the current consensus signal $M_t$. When the round is noisy, $\alpha_t$ becomes small, and \ours retains more of the propagated state. The update is
\begin{align}\label{eq:update-mv}
\mu_t \;=\; \tilde{\mu}_t + \alpha_t(M_t - \tilde{\mu}_t),
\quad
v_t \;=\; (1-\alpha_t)\,\tilde{v}_t.
\end{align}
This recursion has the same form as \textbf{the posterior mean and variance update} of a linear-Gaussian filter with effective noise $R_t$, as formally derived in~\Secref{sec:theory}.



We show the evolution of round-level confidence $\mu_t$ in~\Figref{fig:mu-per-round}. As the debate proceeds, $\mu_t$ gradually increases, indicating that the system becomes more confident over rounds. Moreover, correct samples consistently exhibit higher average $\mu_t$ than incorrect samples, suggesting that the sequential confidence estimate remains informative throughout the interaction.

\begin{figure}[t]
    \centering
    \begin{subfigure}[t]{0.48\linewidth}
        \centering
        \includegraphics[width=\linewidth]{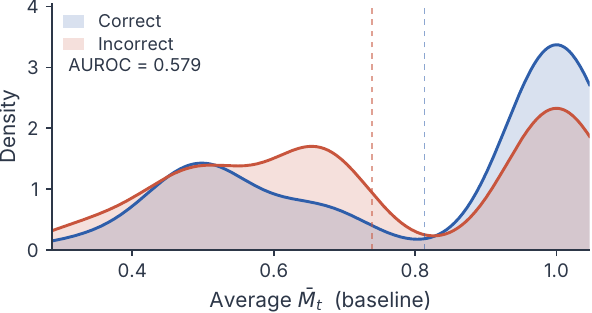}
        \caption{Baseline $\bar M_t$.}
        \label{fig:dist_mt}
    \end{subfigure}
    \hfill
    \begin{subfigure}[t]{0.48\linewidth}
        \centering
        \includegraphics[width=\linewidth]{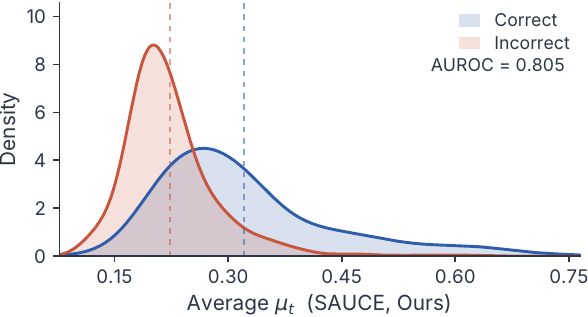}
        \caption{\ours $\mu_t$.}
        \label{fig:dist_spi}
    \end{subfigure}
    \caption{Distribution of (a) the baseline agreement signal $\bar{M}_t$ and (b) \ours's posterior-mean estimate $\mu_t$ (each averaged over rounds) for correct vs.\ incorrect MAS outputs (MMLU-Pro, Gemini-2.5-FL, DyLAN). \ours sharpens the separation between classes (AUROC $0.579 \!\to\! 0.805$). Densities are normalized independently in each panel.}
    \label{fig:mt_vs_spin}
\end{figure}

Finally, \ours aggregates the consensus estimate across the trajectory into a per-record uncertainty score:
\begin{align}\label{eq:final-u}
u(\vx,\tau) \;=\; -\frac{1}{T}\sum\nolimits_{t=1}^{T} \mu_t.
\end{align}
We use the temporal average rather than the terminal value $\mu_T$ so that the score reflects the full deliberation process. 

\textbf{Remark.}\quad
As shown in~\Figref{fig:mt_vs_spin}, directly averaging $M_t$ yields mixed and highly overlapping distributions for correct and incorrect trajectories, whereas \ours produces a much cleaner separation.
This highlights the importance of sequentially modeling how confidence evolves across rounds.

\section{Theoretical Analysis}
\label{sec:theory}
We provide a probabilistic interpretation of \ours by viewing multi-agent interaction as a sequential inference process over a latent system belief. 
This perspective is inspired by the hierarchical Bayesian deep learning framework~\cite{BDL,BDLSurvey}, where observable neural representations are connected to latent probabilistic variables through a hierarchical probabilistic graphical model~\cite{VI}. 
As illustrated in~\Figref{fig:prob-spin}, we treat agent outputs at each round as noisy observations of a latent belief and estimate uncertainty through posterior inference.

\begin{wrapfigure}{r}{0.5\linewidth}
    \centering
    \vspace{-1.0em}
    \includegraphics[width=\linewidth]{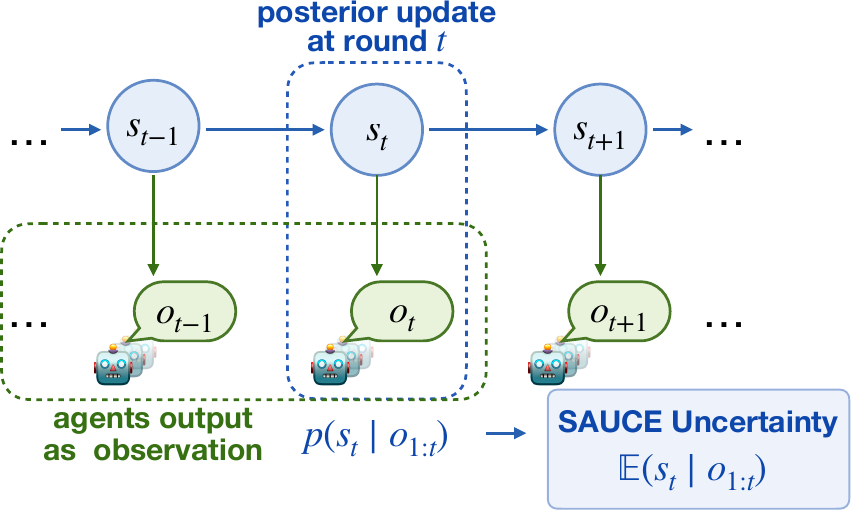}
    \vspace{-1.6em}
    \caption{Probabilistic view of \ours. A latent system belief $s_t$ evolves across rounds and emits a noisy observation $o_t$ formed from the agents' outputs. \ours's per-round estimate is the posterior expectation $\mathbb{E}[s_t\mid o_{1:t}]\!=\!\mu_t$.}
    \label{fig:prob-spin}
\end{wrapfigure}

\subsection{Sequential Inference over Latent System Beliefs}
Let $s_t \in \mathbb{R}$ denote the latent belief of the multi-agent system at round $t$, representing the system's underlying belief about the correctness of its current solution. 

\textbf{Transition Model.}\quad
We assume that the latent belief evolves over rounds according to a first-order Markov process~\cite{rabiner1989tutorial,murphy2012machine, sarkka2023bayesian},
\begin{align}
p(s_t \mid s_{t-1}),
\end{align}
which captures the temporal evolution of the system belief during interaction.

\textbf{Observation Model.}\quad
At round $t$, we extract a round-level signal $o_t$ from the collection of agent outputs and treat it as a noisy observation of the latent belief state $s_t$. Formally, the observation process is described by
\begin{align}
p(o_t \mid s_t),
\end{align}
which specifies the conditional distribution of the observed signal given the latent belief.

\textbf{Sequential Inference.}\quad
Given the transition and observation models, posterior inference over the latent belief follows the Bayesian filtering~\cite{sarkka2023bayesian, doucet2001introduction, kalman1960new} recursion
\begin{align}
p(s_t \mid o_{1:t})
\propto
p(o_t \mid s_t)\int p(s_t \mid s_{t-1})\, p(s_{t-1} \mid o_{1:t-1})\, ds_{t-1}.
\end{align}

\subsection{Linear-Gaussian Instantiation and Posterior Variance Update}
To obtain a tractable update rule, we adopt a \emph{linear-Gaussian} instantiation of the framework above. Specifically, we assume the latent belief evolves as
\begin{align}
\label{eq:latent-evolve}
s_t = \lambda\, s_{t-1} + \epsilon_t,
\qquad \epsilon_t \sim \mathcal{N}(0, q),
\end{align}
where $\lambda \in (0,1)$ controls the temporal dependence across rounds and $q$ is the process noise that captures additional uncertainty introduced during interaction. The latent belief therefore evolves smoothly while being continually shaped by the accumulated interaction history. Treating $\lambda$ as a constant is a deliberate simplification, which we justify empirically through ablations in~\appref{app:constant-lambda}.

\begin{proposition}
\label{prop:update-gaussian}
Under the linear-Gaussian instantiation above, let the observation be $o_t \mid s_t \sim \mathcal{N}(s_t, r_t^2)$. If the posterior belief remains Gaussian, $s_t \mid o_{1:t} \sim \mathcal{N}(\mu_t, \sigma_t^2)$, then the posterior mean and variance satisfy
\begin{align}
\mu_t      &= \tilde{\mu}_t + \alpha_t \bigl(o_t - \tilde{\mu}_t\bigr), \\
\sigma_t^2 &= \left(\frac{1}{\tilde{\sigma}_t^2} + \frac{1}{r_t^2}\right)^{-1}
            = (1-\alpha_t)\,\tilde{\sigma}_t^2,
\end{align}
where the predictive moments and Kalman gain~\cite{kalman1960new} are
$\tilde{\mu}_t = \lambda\, \mu_{t-1}$,
$\tilde{\sigma}_t^2 = \lambda^2 \sigma_{t-1}^2 + q$, and
$\alpha_t = \tilde{\sigma}_t^2 / (\tilde{\sigma}_t^2 + r_t^2)$.
See~\appref{app:proof-prop-update} for the proof.
\end{proposition}

\Propref{prop:update-gaussian} shows that the \ours recursion (\Eqref{eq:trans-mv}--\Eqref{eq:update-mv}) is the exact posterior of this linear-Gaussian model under the identification $o_t \!\leftrightarrow\! M_t$ and $r_t^2 \!\leftrightarrow\! R_t$. Substituting these into the proposition recovers the \ours update verbatim, with the agreement signal $M_t$ acting as a noisy observation of the latent belief and the generation-uncertainty signal $R_t$ as its observation-noise variance. The running estimate $\mu_t$ is therefore the posterior mean $\mathbb{E}[s_t \mid o_{1:t}]$, and the stability weight $v_t$ its posterior variance $\mathrm{Var}[s_t \mid o_{1:t}]$.

\textbf{Takeaway.}\quad \ours carries two quantities with a clean probabilistic reading. The round-level estimate $\mu_t$ \emph{is} the posterior mean of the latent system belief $s_t$, so the per-record score $u(\vx,\tau) = -\frac{1}{T}\sum_t \mu_t$ summarizes how the system's belief evolves across the entire trajectory rather than relying on the terminal round alone. 

\section{Experiments}
\label{sec:experiments}
We evaluate \ours on two standard uncertainty-estimation tasks for large language models:
\emph{(i)} \textbf{misclassification detection} and \emph{(ii)} \textbf{selective prediction}.
We further analyze \textbf{calibration} via the Brier score after Platt scaling, and include two diagnostic studies on long-debate robustness and the contribution of the sequential update.
\subsection{Experimental Setup}
\label{sec:setup}

\textbf{Models.}\quad
We evaluate five backbones spanning both open- and closed-source models. 
The open-source models include \textbf{Qwen3-4B}~\cite{yang2025qwen3}, \textbf{Llama-3.1-8B-Instruct}~\cite{grattafiori2024llama}, and \textbf{Gemma-4-E4B-it}~\cite{google2026gemma4e4bit}; the closed-source models include \textbf{Gemini-2.5-Flash-Lite}~\cite{comanici2025gemini} and \textbf{GPT-5.4-mini}~\cite{openai2026gpt54mini}.
Open-source models are served locally with vLLM~\cite{kwon2023efficient}, which provides per-token log-probabilities~\cite{murray2018correcting, fu2025deep}; for closed-source models, we use the top-$k$ log-probabilities returned by the API.

\textbf{MAS Architectures.}\quad
We run two representative MAS protocols on every backbone. \emph{(i)} \textbf{Debate}~\cite{du2024improving} uses $N\!=\!3$ agents that first solve each problem independently and then exchange their full responses for $T\!=\!3$ rounds. \emph{(ii)} \textbf{DyLAN}~\cite{liu2023dynamic} places the same agents in a dynamic agent network with a listwise ranker that retains the top $k\!=\!2$ replies each round and an early-stopping consensus threshold of $0.667$. For the main experiments, both protocols use up to $T\!=\!3$ rounds; the consensus answer is the majority vote of the surviving agents in the final round. We probe longer debates separately in~\Secref{sec:scale-rounds}.

\textbf{Benchmarks.}\quad
We evaluate on five widely used reasoning suites covering mathematical problem solving, broad knowledge, and hard instruction-following.
For \emph{open-source} models we use \textbf{MATH-500}~\cite{lightman2024let}, \textbf{MMLU-Pro}~\cite{wang2024mmlu}, and \textbf{BBH}~\cite{suzgun2022challenging}.
For \emph{closed-source} models we use \textbf{OlympiadBench}~\cite{he2024olympiadbench}, \textbf{MMLU-Pro} (subset), and \textbf{BBEH-mini}~\cite{kazemi2025big}. Benchmark sizes, formats, and answer verifiers are summarized in~\appref{app:benchmarks}.

\textbf{Baselines.}\quad
We compare \ours with four training-free uncertainty estimators:
\textbf{PE}, the predictive entropy of the final-round output~\cite{malinin2020uncertainty};
\textbf{LL}, the log-likelihood of the final answer~\cite{murray2018correcting};
\textbf{UDPO}~\cite{tang2026value}, a system-level variance score for debate; and
\textbf{MATU}~\cite{chen2026every}, an uncertainty estimator designed for multi-agent systems. The original MATU samples multiple MAS runs at substantially higher inference cost, so for budget parity we use its single-run variant ($N\!=\!1$).

\textbf{Metrics.}\quad 
Following standard practice in uncertainty estimation~\cite{farquhar2024detecting, chen2026every, TokUR}, we measure ranking quality with \textbf{AUROC}~$\uparrow$, which captures how well an estimator separates correct from incorrect MAS outputs, and \textbf{AUARC}~$\uparrow$~\cite{hanley1982meaning, boyd2013area}, which summarizes the accuracy--coverage curve obtained by rejecting increasingly uncertain inputs. For calibration, we report the \textbf{Brier score after Platt scaling}~\cite{glenn1950verification, platt1999probabilistic, guo2017calibration} with 5-fold cross-validation (lower is better). All three metrics are reported as percentages ($\times 100$). Formal definitions are given in~\appref{app:metrics}.

\textbf{Implementation Details.}\quad
All experiments use 3 agents and 3 interaction rounds, decoding with temperature $0.7$, top-$p\!=\!0.9$, and an 8192-token budget per agent. \ours uses transferred hyperparameters: for each MAS architecture, we tune on a single calibration cell and reuse the same values across every other model--dataset combination. The transferred values are listed in~\appref{app:hp-values}, and per-axis sensitivity is reported in~\appref{app:hyperparams}.
\begin{table*}[t]
    \caption{Misclassification detection (AUROC ~$\uparrow$) and selective prediction (AUARC~$\uparrow$) on open-source models.
    \textbf{Boldface} and \underline{underlining} denote the best and second-best performance.}
    \vspace{-1em}
    \label{tab:main-open-spi-ue}
    \begin{center}
    \resizebox{1\linewidth}{!}{%
    \setlength{\tabcolsep}{3.5pt}
    \begin{tabular}{l cccc @{\hspace{10pt}} cccc @{\hspace{10pt}} cccc @{\hspace{10pt}} cc}
    \toprule
      \multirow{3}{*}{\textbf{\makecell{Method}}}
      & \multicolumn{4}{c}{\textbf{MATH-500}}
      & \multicolumn{4}{c}{\textbf{MMLU-Pro}}
      & \multicolumn{4}{c}{\textbf{BBH}}
      & \multicolumn{2}{c}{\multirow{2}{*}{\textbf{Avg}}} \\
      \cmidrule(lr){2-5} \cmidrule(lr){6-9} \cmidrule(lr){10-13}
      & \multicolumn{2}{c}{Debate}
      & \multicolumn{2}{c}{DyLAN}
      & \multicolumn{2}{c}{Debate}
      & \multicolumn{2}{c}{DyLAN}
      & \multicolumn{2}{c}{Debate}
      & \multicolumn{2}{c}{DyLAN}
      & & \\
      \cmidrule(lr){2-3} \cmidrule(lr){4-5} \cmidrule(lr){6-7} \cmidrule(lr){8-9} \cmidrule(lr){10-11}
  \cmidrule(lr){12-13} \cmidrule(lr){14-15}
      & {ROC} & {ARC}
      & {ROC} & {ARC}
      & {ROC} & {ARC}
      & {ROC} & {ARC}
      & {ROC} & {ARC}
      & {ROC} & {ARC}
      & {ROC} & {ARC} \\

    \midrule
    \multicolumn{15}{c}{\textbf{Qwen3-4B}} \\
    \midrule
    PE
        & 64.89
        & \underline{83.45}
        & 64.58
        & 82.70
        & 59.54
        & 73.63
        & 63.53
        & 74.92
        & 68.90
        & 87.13
        & 59.13
        & 83.79
        & 63.43
        & 80.94 \\
    LL
        & \underline{65.93}
        & 83.14
        & 66.36
        & 84.08
        & \underline{68.39}
        & \underline{78.49}
        & \underline{69.88}
        & \underline{79.27}
        & \underline{69.56}
        & \underline{87.62}
        & 60.61
        & 83.82
        & \underline{66.79}
        & \underline{82.74} \\
    UDPO
        & 57.29
        & 79.14
        & \underline{72.50}
        & \underline{84.10}
        & 58.22
        & 66.43
        & 67.95
        & 73.43
        & 59.18
        & 84.38
        & \underline{64.97}
        & \textbf{88.00}
        & 63.35
        & 79.25 \\
    MATU (N=1)
        & 62.11
        & 82.90
        & 63.30
        & 84.05
        & 58.55
        & 73.11
        & 57.81
        & 73.46
        & 58.59
        & 84.56
        & 63.62
        & 85.76
        & 60.66
        & 80.64 \\
    \rowcolor{oursrow}
    \ours (Ours)
        & \textbf{68.81}
        & \textbf{84.12}
        & \textbf{73.11}
        & \textbf{85.60}
        & \textbf{77.05}
        & \textbf{82.72}
        & \textbf{78.14}
        & \textbf{83.37}
        & \textbf{73.68}
        & \textbf{89.09}
        & \textbf{70.18}
        & \underline{87.93}
        & \textbf{73.50}
        & \textbf{85.47} \\
    \midrule
    \multicolumn{15}{c}{\textbf{Llama-3.1-8B-Instruct}} \\
    \midrule
    PE
        & 35.69
        & 34.10
        & 46.96
        & 43.83
        & 45.90
        & 46.51
        & 58.68
        & 57.33
        & 53.92
        & 58.99
        & 63.45
        & 71.06
        & 50.77
        & 51.97 \\
    LL
        & 39.20
        & 35.68
        & 47.62
        & 46.44
        & 48.69
        & 47.12
        & 57.69
        & 56.99
        & 57.78
        & 61.44
        & 63.66
        & 71.53
        & 52.44
        & 53.20 \\
    UDPO
        & \underline{70.94}
        & \underline{63.13}
        & 70.21
        & 66.81
        & \underline{66.96}
        & \underline{61.75}
        & \underline{68.71}
        & \underline{66.58}
        & \underline{66.71}
        & \underline{71.44}
        & 62.82
        & 72.22
        & \underline{67.72}
        & \underline{66.99} \\
    MATU (N=1)
        & 55.55
        & 49.77
        & \underline{73.85}
        & \underline{70.71}
        & 52.85
        & 53.23
        & 63.26
        & 65.32
        & 61.31
        & 64.20
        & \underline{66.24}
        & \underline{72.51}
        & 62.18
        & 62.62 \\
    \rowcolor{oursrow}
    \ours (Ours)
        & \textbf{83.63}
        & \textbf{72.12}
        & \textbf{80.69}
        & \textbf{74.46}
        & \textbf{75.83}
        & \textbf{71.00}
        & \textbf{74.44}
        & \textbf{72.30}
        & \textbf{74.07}
        & \textbf{74.81}
        & \textbf{72.35}
        & \textbf{76.94}
        & \textbf{76.84}
        & \textbf{73.61} \\
    \midrule
    \multicolumn{15}{c}{\textbf{Gemma-4-E4B-it}} \\
    \midrule
    PE
        & 58.20
        & 87.47
        & \textbf{66.08}
        & \underline{88.45}
        & 69.33
        & 83.64
        & 75.50
        & 87.58
        & 53.67
        & 84.69
        & 54.47
        & 79.97
        & 62.88
        & 85.30 \\
    LL
        & \textbf{62.43}
        & \underline{87.55}
        & \underline{65.12}
        & \textbf{88.58}
        & \underline{74.04}
        & \underline{86.63}
        & \underline{76.19}
        & \underline{87.89}
        & \underline{66.28}
        & 89.59
        & 56.53
        & 80.70
        & \underline{66.77}
        & \underline{86.82} \\
    UDPO
        & 55.83
        & 85.47
        & 55.53
        & 84.45
        & 62.01
        & 75.64
        & 72.28
        & 81.36
        & 64.01
        & \underline{93.83}
        & \underline{80.36}
        & \textbf{92.30}
        & 65.00
        & 85.51 \\
    MATU (N=1)
        & 54.01
        & 86.99
        & 60.64
        & 87.72
        & 55.37
        & 77.41
        & 65.76
        & 82.36
        & 57.73
        & 90.82
        & 74.98
        & 88.82
        & 61.41
        & 85.69 \\
    \rowcolor{oursrow}
    \ours (Ours)
        & \underline{62.29}
        & \textbf{87.95}
        & 64.05
        & 87.89
        & \textbf{82.24}
        & \textbf{89.75}
        & \textbf{82.57}
        & \textbf{90.08}
        & \textbf{83.62}
        & \textbf{96.33}
        & \textbf{83.53}
        & \underline{91.94}
        & \textbf{76.38}
        & \textbf{90.66} \\
    \bottomrule
    \end{tabular}
    }
    \end{center}
\end{table*}
\begin{table*}[t]
    \caption{Misclassification detection and selective prediction on closed-source models. Metrics and notation follow Table~\ref{tab:main-open-spi-ue}. \textbf{Boldface} and \underline{underlining} denote the best and second-best performance.}
    \vspace{-1em}
    \label{tab:main-closed-spi-ue}
    \begin{center}
    \resizebox{1\linewidth}{!}{%
    \setlength{\tabcolsep}{5pt}
    \begin{tabular}{l cccc @{\hspace{10pt}} cccc @{\hspace{10pt}} cccc @{\hspace{10pt}} cc}
    \toprule
    \multirow{3}{*}{\textbf{\makecell{Method}}}
    & \multicolumn{4}{c}{\textbf{Olympiad}}
    & \multicolumn{4}{c}{\textbf{MMLU-Pro}}
    & \multicolumn{4}{c}{\textbf{BBEH-mini}}
    & \multicolumn{2}{c}{\multirow{2}{*}{\textbf{Avg}}} \\
    \cmidrule(lr){2-5} \cmidrule(lr){6-9} \cmidrule(lr){10-13} 
    & \multicolumn{2}{c}{Debate}
    & \multicolumn{2}{c}{DyLAN}
    & \multicolumn{2}{c}{Debate}
    & \multicolumn{2}{c}{DyLAN}
    & \multicolumn{2}{c}{Debate}
    & \multicolumn{2}{c}{DyLAN}
    & & \\
    \cmidrule(lr){2-3} \cmidrule(lr){4-5} \cmidrule(lr){6-7} \cmidrule(lr){8-9} \cmidrule(lr){10-11} \cmidrule(lr){12-13} \cmidrule(lr){14-15}
    & {ROC} & {ARC}
    & {ROC} & {ARC}
    & {ROC} & {ARC}
    & {ROC} & {ARC}
    & {ROC} & {ARC}
    & {ROC} & {ARC}
    & {ROC} & {ARC} \\

    \midrule
    \multicolumn{15}{c}{\textbf{Gemini-2.5-Flash-Lite}} \\
    \midrule
    PE
        & 65.51
        & 66.86
        & 66.35
        & 65.23
        & 48.71
        & 73.35
        & 70.22
        & 89.09
        & 49.98
        & 19.13
        & 41.73
        & 16.06
        & 57.08
        & 54.95 \\
    LL
        & 68.67
        & 68.13
        & 67.42
        & 66.67
        & \underline{64.90}
        & \underline{82.82}
        & \underline{72.18}
        & \underline{90.05}
        & 46.71
        & 18.50
        & 39.41
        & 15.46
        & 59.88
        & 56.94 \\
    UDPO
        & 63.08
        & 48.89
        & \underline{81.58}
        & \underline{71.05}
        & 54.59
        & 77.52
        & 57.83
        & 82.71
        & \underline{59.23}
        & \underline{25.12}
        & \textbf{75.41}
        & \textbf{27.29}
        & \underline{65.29}
        & 55.43 \\
    MATU (N=1)
        & \underline{68.96}
        & \underline{70.73}
        & 64.31
        & 70.23
        & 54.37
        & 78.22
        & 61.06
        & 85.46
        & 56.50
        & 23.83
        & 66.08
        & \underline{25.37}
        & 61.88
        & \underline{58.97} \\
    \rowcolor{oursrow}
    \ours (Ours)
        & \textbf{81.27}
        & \textbf{76.82}
        & \textbf{81.85}
        & \textbf{77.13}
        & \textbf{79.60}
        & \textbf{89.84}
        & \textbf{78.72}
        & \textbf{91.95}
        & \textbf{65.68}
        & \textbf{28.07}
        & \underline{67.48}
        & 24.44
        & \textbf{75.77}
        & \textbf{64.71} \\
    \midrule
    \multicolumn{15}{c}{\textbf{GPT-5.4-mini}} \\
    \midrule
    PE
        & 55.56
        & 54.90
        & 72.58
        & 63.59
        & 54.61
        & 80.26
        & 63.26
        & 84.77
        & 46.36
        & 16.34
        & 39.92
        & 13.00
        & 55.38
        & 52.14 \\
    LL
        & \underline{70.25}
        & \underline{61.87}
        & 75.03
        & \underline{65.69}
        & \underline{61.57}
        & \underline{83.56}
        & 65.51
        & 85.97
        & 41.67
        & 14.18
        & 45.03
        & 13.76
        & 59.84
        & \underline{54.17} \\
    UDPO
        & 65.99
        & 46.61
        & \underline{77.73}
        & 59.55
        & 56.29
        & 81.71
        & \underline{67.70}
        & \underline{86.20}
        & \textbf{58.58}
        & \textbf{19.59}
        & \underline{54.73}
        & \underline{16.08}
        & \underline{63.50}
        & 51.62 \\
    MATU (N=1)
        & 53.45
        & 54.91
        & 65.23
        & 63.74
        & 56.44
        & 81.52
        & 60.31
        & 85.04
        & 48.64
        & 16.96
        & \textbf{55.78}
        & \textbf{18.93}
        & 56.64
        & 53.51 \\
    \rowcolor{oursrow}
    \ours (Ours)
        & \textbf{85.15}
        & \textbf{74.46}
        & \textbf{83.46}
        & \textbf{72.57}
        & \textbf{70.63}
        & \textbf{86.44}
        & \textbf{73.83}
        & \textbf{88.29}
        & \underline{50.08}
        & \underline{17.72}
        & 47.36
        & 14.85
        & \textbf{68.42}
        & \textbf{59.06} \\
    \bottomrule
    \end{tabular}
    }
    \end{center}
        \vspace{-5mm}
\end{table*}
\subsection{Uncertainty Estimation for Multi-Agent Systems}
\label{sec:main-results}
Tables~\ref{tab:main-open-spi-ue} and~\ref{tab:main-closed-spi-ue} report AUROC and AUARC on the full grid of backbones, MAS protocols, and benchmarks of~\Secref{sec:setup}. \ours leads the large majority of cells and every per-backbone average; the gap to the strongest baseline holds under both ranking (AUROC) and accuracy--coverage (AUARC) views and on both Debate and DyLAN. Among the baselines, single-LLM scores (PE, LL) frequently land near or below chance, and UDPO and MATU each lead only on isolated (model, protocol) cells.

\textbf{Results.}\quad Single-round token statistics break down under MAS: interaction reshapes the very entropy and log-likelihood that PE and LL rely on (\Secref{sec:scale-rounds}). Specialized MAS estimators tie their gains to a single protocol and do not transfer across backbones or datasets. Only \ours's filtering-style aggregation across rounds wins consistently across backbones, datasets, and protocols, and the same property lets it carry over to the closed-source setting where only top-$k$ token log-probabilities are exposed.
\vspace{-2mm}
\subsection{Uncertainty Estimation for Calibration}
\begin{figure}[t]
    \centering
    \includegraphics[width=\linewidth]{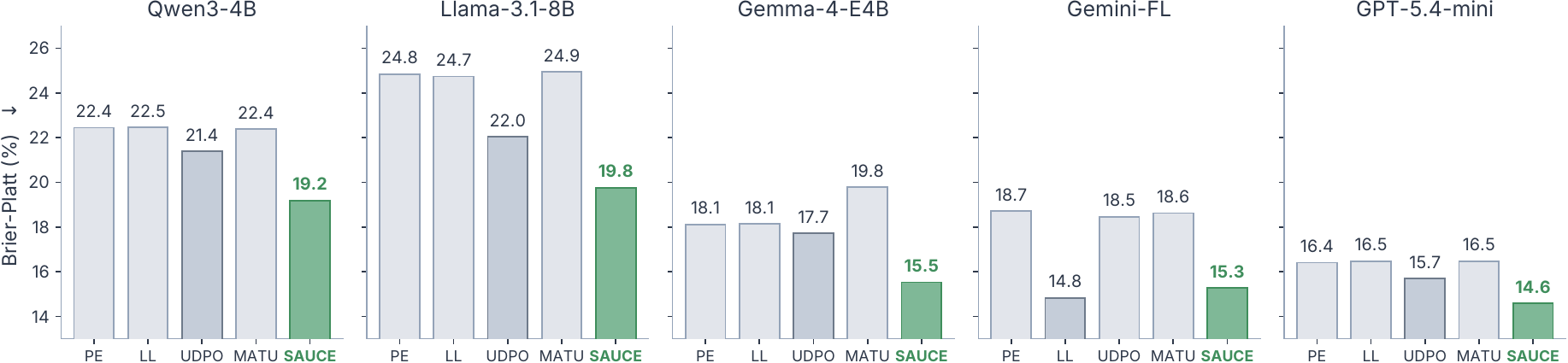}
    \caption{Brier score after Platt scaling (5-fold cross-validation) on MMLU-Pro under the Debate protocol, across all five backbones. \ours attains the lowest Brier-Platt score on four of the five backbones.}
    \label{fig:calibration}
\end{figure}

\begin{figure}[t]
    \centering
    \includegraphics[width=\linewidth]{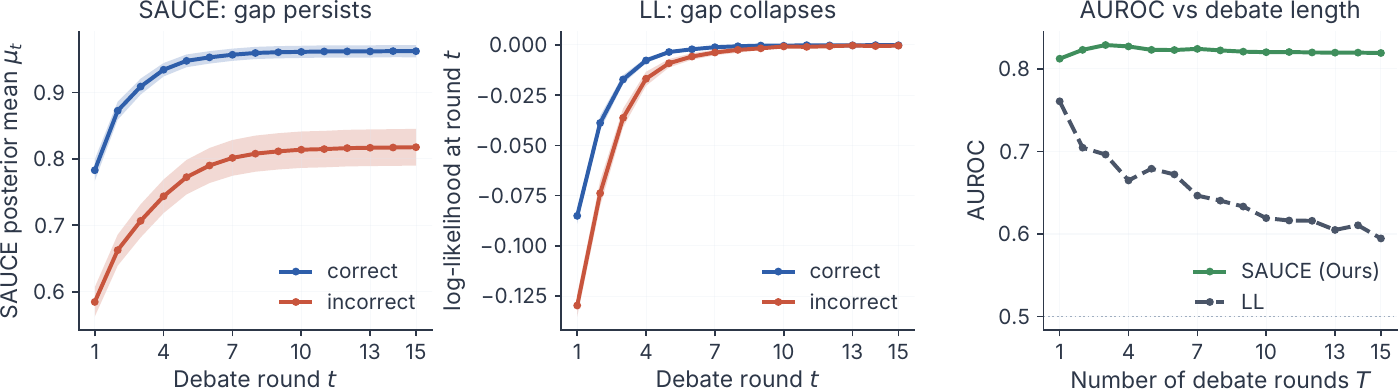}
    \caption{Per-round behavior on a 15-round Qwen3-4B Debate run on MMLU-Pro, grouped by final-answer correctness. 
    \emph{\textbf{Left}:} \ours posterior mean $\mu_t$; the correct/incorrect gap persists through round 15. 
    \emph{\textbf{Middle}:} Average per-agent log-likelihood at round $t$; both classes become increasingly token-confident and the gap vanishes. 
    \emph{\textbf{Right}:} AUROC against final-consensus correctness as a function of debate length $T$; \ours plateaus while LL degrades toward chance (dashed line).}
    \label{fig:scale-rounds}
    \vspace{-5mm}
\end{figure}

\Figref{fig:calibration} reports the Brier score after Platt scaling on MMLU-Pro under the Debate protocol across all five backbones. \ours attains the lowest Brier-Platt score on four of the five backbones; on Gemini-FL it is within $0.5$\,pp of the strongest baseline (LL). Calibration results on additional benchmarks and settings are deferred to~\appref{app:calib}.

\textbf{Results.}\quad Brier penalizes both miscalibrated magnitudes and unresolved predictions, so this view complements AUROC/AUARC: a method can rank well yet pay a Brier penalty when its confidences cluster in a narrow band---trivially calibrated, but unable to separate reliable from unreliable trajectories. \ours's calibrated probabilities, by contrast, remain well spread across the probability range while still landing close to empirical accuracy. The score therefore has the right \emph{shape} as well as the right \emph{order}, and a two-parameter Platt transform suffices to turn it into a per-trajectory confidence usable as a downstream gating signal.


\vspace{-1mm}
\subsection{Robustness in Long Debates}
\label{sec:scale-rounds}


Our main experiments use $T\!=\!3$. To test whether \ours remains reliable in longer interactions, we run Qwen3-4B Debate on MMLU-Pro up to $T\!=\!15$ and compare \ours with a last-round log-likelihood (LL) baseline at each truncation. As the debate proceeds, agents become increasingly token-confident on both correct and incorrect trajectories, causing the LL scores to concentrate and the correct/incorrect gap to vanish (\Figref{fig:scale-rounds}, \emph{middle}). In contrast, \ours's posterior mean $\mu_t$ also rises as agreement accumulates, but preserves a clear class separation throughout the sweep (\emph{left}); correspondingly, \ours's AUROC plateaus after $T\!=\!3$, whereas LL steadily degrades toward chance (\emph{right}). These results suggest that retaining earlier-round evidence makes \ours more robust than last-round confidence scores when debate length increases.

\Needspace{6\baselineskip}
\subsection{Contribution of the Sequential Update}
\label{sec:ablation-signal}
\begin{wraptable}{r}{0.45\linewidth}
\vspace{-1.2em}
\centering
\caption{Component-wise ablation on Qwen3-4B / MMLU-Pro / Debate. Row~(b) is the static joint of $M_t$ and $R_t$; row~(c) is the full sequential update with $\bar\mu$ readout. AUROC~$\uparrow$ and Brier-Platt~$\downarrow$.}
\label{tab:ablation-signal}
\resizebox{\linewidth}{!}{
\setlength{\tabcolsep}{6pt}
\begin{tabular}{lcc}
\toprule
Method & AUROC & Brier-Platt \\
\midrule
Last round, $1{-}M_t$ & 58.4 & 21.3 \\
(a) \quad Round-mean of $M_t$ & 68.6 & 20.0 \\
(b) \quad Static joint with $R_t$  & 68.4 & 20.9 \\
\rowcolor{oursrow}
(c) \quad \textbf{\ours} (Ours) & \textbf{77.1} & \textbf{19.2} \\
\bottomrule
\end{tabular}
}
\vspace{-10mm}
\end{wraptable}

A natural concern is that \ours's improvement may come simply from combining agreement $M_t$ with generation uncertainty $R_t$, rather than from the sequential update itself. 
Row~(b) of Table~\ref{tab:ablation-signal} tests this possibility with a static precision-weighted aggregation $\frac{\sum_t M_t/R_t}{\sum_t 1/R_t}$.

\textbf{Results.}\quad 
Table~\ref{tab:ablation-signal} shows that the static joint score performs closely to the $M_t$-only round mean, while the full \ours update yields a clear gain. 
This indicates that \ours's improvement comes from the sequential update rather than a static combination of $M_t$ and $R_t$.
To further probe the contribution of interaction history, we compare simpler trajectory summaries and alternative round orders (\appref{app:trajectory-summaries}; statistical analyses in \appref{app:statistical-significance}) and examine whether answer-switch histories remain informative when final agreement is held fixed (\appref{app:path-dynamics}).

\section{Related Work}
\textbf{Uncertainty Estimation for LLMs.}\quad
Uncertainty estimation~\cite{NPN,VIR} has been widely studied for hallucination detection and calibration in LLMs~\cite{farquhar2024detecting, liu2024uncertainty, bakman2025reconsidering, liu2025uncertainty, vashurin2025benchmarking}. 
Existing methods include sampling-based approaches, such as Semantic Entropy~\cite{farquhar2024detecting} and SAR~\cite{duan2024shifting}, which estimate uncertainty from multiple generations and their semantic variation; Bayesian approaches, which approximate posterior uncertainty through weight perturbations or low-rank adaptation~\cite{TokUR, shitraining, wang2024blob, yangbayesian}; and internal-signal methods, which use token entropy, log-likelihood, or related generation statistics with little additional cost~\cite{fu2025deep,van2022mutual,ren2022out,darrin2023rainproof, duan2024shifting}. 
While effective for single LLMs, these methods do not directly address the structure of multi-agent reasoning, where uncertainty depends not only on individual generations but also on how agents interact and evolve across rounds. 
Our work addresses this gap by formulating uncertainty estimation in parallel multi-agent reasoning systems as a sequential inference problem over the interaction trajectory.

\textbf{Parallel Multi-Agent Reasoning Systems.}\quad
LLM-based multi-agent systems (MAS) have emerged as a promising paradigm for complex reasoning, where multiple agents collaborate, debate, or critique one another~\cite{talebirad2023multi, li2024survey, shi2025ssr, zhang2024cut}. 
We focus on \emph{parallel multi-agent reasoning systems}, such as debate-style methods~\cite{choidebate, du2024improving, cui2025free} and DyLAN~\cite{liu2023dynamic}, where agents solve the same input at each round, exchange information, and produce a verifiable final prediction. 
This setting naturally exposes system-level signals~\cite{ohuncertainty, tang2026value}, such as cross-agent agreement, disagreement, and their evolution over rounds. 
Recent work has begun to study uncertainty in MAS, including general frameworks for uncertainty quantification in LLM agents~\cite{ohuncertainty}, uncertainty-driven optimization for mitigating debate collapse~\cite{tang2026value}, and MATU~\cite{chen2026every}, which estimates trajectory-level uncertainty from repeated MAS runs. 
In contrast, \ours is a lightweight, training-free estimator that updates uncertainty online from standard rollout signals within a single parallel multi-agent reasoning trajectory.
\vspace{-2mm}
\section{Conclusion}
We studied uncertainty estimation in parallel multi-agent reasoning systems. We formulated MAS uncertainty as a sequential inference problem over the multi-round interaction trajectory, and proposed \ours: a filtering-style update that propagates a latent consensus state across rounds. Across Debate and DyLAN protocols on multiple reasoning benchmarks, \ours consistently outperforms log-likelihood baselines and recent MAS-specific uncertainty estimators on both ranking and calibration, and remains stable as debates extend over more rounds where last-round confidence scores collapse. More broadly, our results suggest that interaction dynamics are a useful primitive for reasoning about the reliability of multi-agent systems, opening the way toward uncertainty-aware orchestration and adaptive control of MAS.

\textbf{Limitations.}\quad 
\ours focuses on \emph{parallel multi-agent reasoning systems}, where agents produce comparable per-round answers. This assumption limits its direct applicability to more general agentic systems with heterogeneous roles, tool use, or open-ended planning, where round-level consensus is not naturally defined. Our current instantiation also uses simple agreement and uncertainty signals, namely modal-answer share and token-level predictive entropy; developing principled alternatives remains an important direction for future work.

\section*{Acknowledgement}
\label{sec:acknowledgement}
We thank all reviewers, AC, and SAC for their valuable comments. HW is supported by Amazon Faculty Research Award, Microsoft AI \& Society Fellowship, Nokia Faculty Research Award, NSF CAREER Award IIS-2340125, NIH grant R01CA297832, and NSF grant IIS-2127918. 

\noindent\textbf{Competing Interests.} The authors declare no competing interests.

\bibliographystyle{abbrv}
\bibliography{ref}

\newpage
\appendix
\section*{\LARGE Appendix}
\markboth{Appendix}{Appendix}

In \appref{app:algorithm}, we present the full algorithmic description of \ours as a single-pass procedure over a multi-agent reasoning trajectory.
In \appref{app:impl}, we provide our implementation details, including:
\begin{itemize}[nosep]
    \item \textbf{evaluation metrics}~(\appref{app:metrics}),
    \item \textbf{benchmark formats and subset construction}~(\appref{app:benchmarks}),
    \item and \textbf{hyperparameter configuration}~(\appref{app:hp-values}).
\end{itemize}
In \appref{app:empirical}, we present additional empirical studies, including:
\begin{itemize}[nosep]
    \item a \textbf{comparison with a learned, context-dependent transition coefficient}~(\appref{app:constant-lambda}),
    \item \textbf{per-axis hyperparameter sensitivity}~(\appref{app:hyperparams}),
    \item the \textbf{full Brier-Platt calibration grid}~(\appref{app:calib}),
    \item \textbf{trajectory-summary comparisons and temporal-order ablations}~(\appref{app:trajectory-summaries}),
    \item \textbf{bootstrap comparisons and repeated-rollout variability}~(\appref{app:statistical-significance}),
    \item and \textbf{trajectory analysis under unanimous final agreement}~(\appref{app:path-dynamics}).
\end{itemize}
Finally, in \appref{app:theory}, we provide proofs and derivations supporting our theoretical claims, including:
\begin{itemize}[nosep]
    \item the \textbf{proof of Proposition~\ref{prop:update-gaussian}}~(\appref{app:proof-prop-update}),
    \item the \textbf{innovations form of the marginal likelihood}~(\appref{app:proof-innovations}),
    \item and the \textbf{equivalence with the evidence lower bound}~(\appref{app:elbo}).
\end{itemize}

\section{Algorithm}
\label{app:algorithm}

Algorithm~\ref{alg:spi} summarizes \ours as a single-pass procedure over a parallel-MAS rollout. Given $\tau = (\vz_t^{1:N})_{t=1}^T$, \ours traverses the rounds once, extracts the round-level signals $(M_t, R_t)$ defined in~\Secref{sec:reliability}, and applies the propagation--update recursion of~\Secref{sec:sequential}. The procedure is training-free, requires no additional forward passes, and operates entirely on quantities available in a standard rollout (parsed answers and top-$k$ token log-probabilities). All hyperparameters are scalars and are shared across (model, dataset) cells within each MAS protocol (\appref{app:hyperparams}).

\begin{algorithm}[h]
\caption{\ours: Sequential Agent Uncertainty through Consensus Evolution.}
\label{alg:spi}
\begin{algorithmic}[1]
\Require Rollout $\tau = (\vz_t^{1:N})_{t=1}^T$ for input $\vx$; hyperparameters $\lambda \in (0,1)$, $q > 0$, $\mu_0$, $v_0 > 0$.
\Ensure Per-record uncertainty score $u(\vx, \tau)$.
\State $(\mu, v) \gets (\mu_0, v_0)$
\For{$t = 1, \ldots, T$}
    \State Extract parsed answers $\{a_t^i\}_{i=1}^N$ and per-token log-probabilities from $\{\vz_t^i\}_{i=1}^N$
    \State $m_t \gets \mathrm{mode}(\{a_t^i\}_{i=1}^N)$ \Comment{round modal answer}
    \State $M_t \gets \frac{1}{N}\sum_{i=1}^{N} \mathbb{1}[a_t^i = m_t]$ \Comment{agreement signal (\Eqref{eq:obs})}
    \State $R_t \gets \frac{1}{N}\sum_{i=1}^{N} H_{\mathrm{tok}}(\vz_t^i)$ \Comment{generation uncertainty (\Eqref{eq:R_round})}
    \State $\tilde{\mu} \gets \lambda\,\mu$,\quad $\tilde{v} \gets \lambda^2 v + q$ \Comment{predict (\Eqref{eq:trans-mv})}
    \State $\alpha \gets \tilde{v} / (\tilde{v} + R_t)$ \Comment{Kalman gain (\Eqref{eq:alpha})}
    \State $\mu \gets \tilde{\mu} + \alpha\,(M_t - \tilde{\mu})$,\quad $v \gets (1-\alpha)\,\tilde{v}$ \Comment{update (\Eqref{eq:update-mv})}
    \State $\mu_t \gets \mu$
\EndFor
\State \Return $u(\vx, \tau) = -\dfrac{1}{T}\sum_{t=1}^T \mu_t$ \Comment{(\Eqref{eq:final-u})}
\end{algorithmic}
\end{algorithm}

\section{Implementation Details}
\label{app:impl}

This appendix collects the metric definitions, benchmark formats, and hyperparameter values used in the main results.

\subsection{Metrics}
\label{app:metrics}

Each estimator outputs an \emph{uncertainty score} $s_i = u(\vx_i, \tau_i)$ per record (higher = less reliable), evaluated against a binary correctness label $y_i \in \{0, 1\}$ on a dataset of size $N$. Let $\mathcal{P} = \{i : y_i = 1\}$ and $\mathcal{N} = \{i : y_i = 0\}$.

\textbf{AUROC.}\quad The probability that a uniformly drawn correct trajectory has a lower uncertainty score than a uniformly drawn incorrect one, with ties contributing one half:
\begin{align}
\mathrm{AUROC} \;=\; \frac{1}{|\mathcal{P}||\mathcal{N}|} \sum_{i \in \mathcal{P}}\sum_{j \in \mathcal{N}} \Bigl(\mathbb{1}[s_i < s_j] + \tfrac{1}{2}\,\mathbb{1}[s_i = s_j]\Bigr),
\end{align}
equivalent to the Mann--Whitney $U$ statistic, with tied pairs contributing $\tfrac{1}{2}$.

\textbf{AUARC.}\quad The area under the accuracy--coverage curve obtained by rejecting the most uncertain trajectories first. Let $\pi$ be the permutation that orders records by ascending uncertainty, $s_{\pi(1)} \le s_{\pi(2)} \le \ldots \le s_{\pi(N)}$. Then
\begin{align}
\mathrm{AUARC} \;=\; \frac{1}{N}\sum_{k=1}^{N} \frac{1}{k}\sum_{i=1}^{k} y_{\pi(i)},
\end{align}
i.e., the average accuracy on the top-$k$ retained subset over all coverage levels $k/N$.

\textbf{Brier after Platt scaling.}\quad The Brier score evaluated on calibrated probabilities,
\begin{align}
\mathrm{Brier} \;=\; \frac{1}{N}\sum_{i=1}^{N} \bigl(\hat p_i - y_i\bigr)^2,
\qquad
\hat p_i \;=\; \sigma\!\bigl(a \cdot (-s_i) + b\bigr),
\end{align}
where Platt parameters $(a, b)$ are fit by minimising the binary cross-entropy on training data. We use 5-fold cross-validation: $(a, b)$ is refit on each training fold, the Brier score is computed on the corresponding held-out fold, and we report the average across folds.

All three metrics are reported as percentages ($\times 100$).

\subsection{Benchmarks}
\label{app:benchmarks}

We follow each benchmark's standard test split or, where listed in \Tabref{tab:benchmarks-summary}, a fixed random subset shared across all backbones. The same prompt template and decoding configuration are used for every (model, dataset, MAS) cell. The consensus answer of a MAS rollout is judged by an automatic verifier matched to the answer format:

\textbf{Benchmark Choice.}\quad Open-source models in our suite (Qwen3-4B, Llama-3.1-8B-Instruct, Gemma-4-E4B-it) are at the smaller-scale end of LLMs and already exhibit a wide accuracy spread on MATH-500, MMLU-Pro, and BBH, yielding a well-balanced split of correct and incorrect MAS outputs that uncertainty estimation can be meaningfully evaluated against. Closed-source models (Gemini-2.5-Flash-Lite, GPT-5.4-mini) largely saturate the original MATH-500 and BBH, leaving too little correct/incorrect variance for the same purpose. We therefore replace those with the harder variants, e.g., OlympiadBench for mathematics and BBEH-mini for hard instruction-following, to restore an informative spread at this stronger model class. MMLU-Pro spans a wide enough difficulty range to remain useful for both classes; we use the full split for open-source models and the $5{,}000$ subset above for closed-source models to control API cost.

\textbf{Multiple-Choice} (MMLU-Pro, BBH, BBEH-mini): canonicalised exact match on the parsed option after format-aware normalisation (whitespace, parenthesisation, casing).

\textbf{Mathematical Reasoning} (MATH-500, OlympiadBench): symbolic equivalence via the Math-Verify framework, which handles algebraic simplification (e.g.\ fraction/decimal forms, sign).

\textbf{Subset Construction.}\quad For closed-source models we use a $5{,}000$-question random subset of the MMLU-Pro test split, sampled uniformly from its $12{,}032$ items with a fixed seed and shared across all closed-source backbones and MAS protocols. Drawing a subset (rather than the full split) keeps API inference cost manageable while preserving the broad subject coverage of MMLU-Pro. BBEH-mini is used as released by~\cite{kazemi2025big}, a compact $\approx\!460$-question collection spanning $23$ Big-Bench Extra Hard tasks; we use it without further filtering.

\begin{table}[h]
\centering
\caption{Benchmark formats and evaluation sizes used in the main results. ``--'' denotes a benchmark not used for that model class.}
\label{tab:benchmarks-summary}
\setlength{\tabcolsep}{8pt}
\begin{tabular}{l l r r}
\toprule
Dataset       & Format             & Open-source size & Closed-source size \\
\midrule
MATH-500      & Free-form math     & $500$       & --                  \\
MMLU-Pro      & Multiple-choice    & $12{,}032$  & $5{,}000$ (subset)  \\
BBH           & MCQ / short string & $1{,}250$   & --                  \\
OlympiadBench & Free-form math     & --          & $674$               \\
BBEH-mini     & MCQ / short string & --          & $\approx\!460$      \\
\bottomrule
\end{tabular}
\end{table}

\subsection{Hyperparameter Configuration}
\label{app:hp-values}

\ours uses a single tuple $(\lambda, q, \sigma_0^2, \mu_0)$ per MAS protocol, transferred across every (model, dataset) cell of Tables~\ref{tab:main-open-spi-ue}--\ref{tab:main-closed-spi-ue}. We tune on a single calibration cell (Qwen3-4B / MMLU-Pro) by joint grid search; per-axis sensitivity around these defaults is reported in \appref{app:hyperparams}.

\begin{table}[h]
\centering
\caption{\ours hyperparameter defaults transferred across all cells.}
\label{tab:hp-config}
\setlength{\tabcolsep}{12pt}
\begin{tabular}{l cccc}
\toprule
MAS protocol & $\lambda$ & $q$ & $\sigma_0^2$ & $\mu_0$ \\
\midrule
Debate & $0.95$  & $0.10$  & $0.25$ & $0.0$ \\
DyLAN  & $0.70$  & $0.001$ & $0.25$ & $0.0$ \\
\bottomrule
\end{tabular}
\end{table}

\section{Additional Experiments}
\label{app:empirical}

This appendix collects additional empirical studies supporting the main results: a comparison between constant and context-dependent transition coefficients (\appref{app:constant-lambda}), hyperparameter sensitivity (\appref{app:hyperparams}), the full calibration grid (\appref{app:calib}), comparisons with simpler trajectory summaries and temporal-order ablations (\appref{app:trajectory-summaries}), statistical analyses (\appref{app:statistical-significance}), and a diagnostic study of trajectories with identical final-round agreement (\appref{app:path-dynamics}).

\subsection{Constant vs.\ Context-Dependent Transition Coefficient}
\label{app:constant-lambda}

The main text instantiates the latent belief as
\begin{align}
s_t = \lambda\, s_{t-1} + \epsilon_t, \qquad \epsilon_t \sim \mathcal{N}(0, q),
\end{align}
with a constant transition coefficient $\lambda$. A natural generalisation is to let $\lambda$ depend on the interaction history $\vh_t$, since how strongly an agent should carry its prior belief into the next round may plausibly vary with the conversation so far. This appendix shows the extra capacity is not exercised: under a tight variational objective, a learnable $\lambda_\theta(\vh_t) = \sigma\!\big(\mathrm{MLP}(\vh_{t-1})\big)$ collapses to a near-degenerate constant whose downstream AUROC is close to that of a fixed scalar, so we adopt the constant $\lambda^{*}$ as the paper's main configuration.

\paragraph{Training Objective: An ELBO.}
We learn $\theta$ by maximising the evidence lower bound (ELBO) on the marginal likelihood of the observation sequence, an objective that does not use correctness labels. For any variational posterior $q(s_{1:T})$,
\begin{align}
\log p_\theta(o_{1:T}) &= \mathcal{L}(\theta, q) + \mathrm{KL}\!\big(q(s_{1:T}) \,\|\, p_\theta(s_{1:T} \mid o_{1:T})\big), \\
\mathcal{L}(\theta, q) &= \mathbb{E}_{q}\!\left[\log p_\theta(o_{1:T}, s_{1:T}) - \log q(s_{1:T})\right].
\end{align}
Under the linear-Gaussian model the Kalman filter induced by $\lambda_\theta$ yields the \emph{exact} posterior, so the bound is tight ($\mathrm{KL}=0$) and reduces to the closed-form innovations decomposition
\begin{align}\label{eq:innovations-ll}
\mathcal{L}(\theta) = -\frac{T}{2}\log(2\pi) - \frac{1}{2}\sum_{t=1}^T \left[\log S_t + \frac{e_t^2}{S_t}\right],
\quad
e_t = o_t - \lambda_\theta(\vh_t)\,\mu_{t-1},
\quad
S_t = \tilde{\sigma}_t^2 + R_t.
\end{align}
The full derivation, including the equivalence between the ELBO and the innovations form, is given in \appref{app:elbo}. Maximising \Eqref{eq:innovations-ll} is therefore variational learning of the transition coefficient, with no looser surrogate to blame if the procedure fails to extract context-dependent structure.

\paragraph{Downstream Comparison.}
\begin{table}[h]
\centering
\caption{Learned vs.\ constant transition coefficient on Qwen3-4B / Debate (AUROC$\uparrow$, $\times 100$). The two filter variants achieve similar AUROC on both datasets; both improve substantially over the no-filter baseline.}
\label{tab:uncertainty-coeff}
\setlength{\tabcolsep}{8pt}
\begin{tabular}{lcc}
\toprule
Method & MATH-500 & MMLU-Pro \\
\midrule
mean$(R_t)$ -- no filter             & 77.00 & 64.93 \\
Learned $\lambda_\theta(\vh_t)$      & 82.85 & \textbf{75.18} \\
Const.\ $\lambda^{*}$ (ours)         & \textbf{82.94} & 74.76 \\
\bottomrule
\end{tabular}
\end{table}
We optimise \Eqref{eq:innovations-ll} end-to-end with respect to $\theta$ on Qwen3-4B Debate trajectories from MATH-500 and MMLU-Pro and compare against the constant $\lambda^{*}$ on the same trajectories. \Tabref{tab:uncertainty-coeff} reports downstream AUROC: the learned context-dependent $\lambda_\theta(\vh_t)$ and constant $\lambda^{*}$ differ by only $0.09$ and $0.42$ AUROC points on MATH-500 and MMLU-Pro, respectively, and both substantially outperform the no-filter baseline. Because the bound is tight, the absence of an improvement cannot be attributed to a loose variational surrogate, and the optimiser nonetheless fails to extract any context-dependent signal that downstream uncertainty estimation can use. Inspection of the trained network confirms this picture: the learned $\lambda_\theta(\vh_t)$ concentrates on a near-degenerate spike with no measurable conditioning on correctness. We therefore adopt the constant transition coefficient as the paper's main configuration: it achieves similar AUROC to the learned variant while remaining fully training-free, and it side-steps the additional engineering overhead (hidden-state access, optimiser hyperparameters, validation splits) that the learnable alternative drags along.

\paragraph{Implementation Details.}
\label{app:lambda-impl}
The MLP has a 64-dimensional bottleneck and a 32-dimensional ReLU hidden layer; $q$, $\mu_0$ and $\sigma_0^2$ are fixed buffers shared with the constant baseline. We optimise the innovations NLL (the negative of \Eqref{eq:innovations-ll}) with Adam at learning rate $10^{-3}$ and early-stop on validation AUROC, which uses correctness labels. Trajectories are split $70/15/15$ into train, validation and test; sample sizes are $n\!=\!500$ for MATH-500 and $n\!=\!12{,}032$ for MMLU-Pro.

\subsection{Hyperparameter Sensitivity}
\label{app:hyperparams}

\begin{figure}[h]
    \centering
    \includegraphics[width=\linewidth]{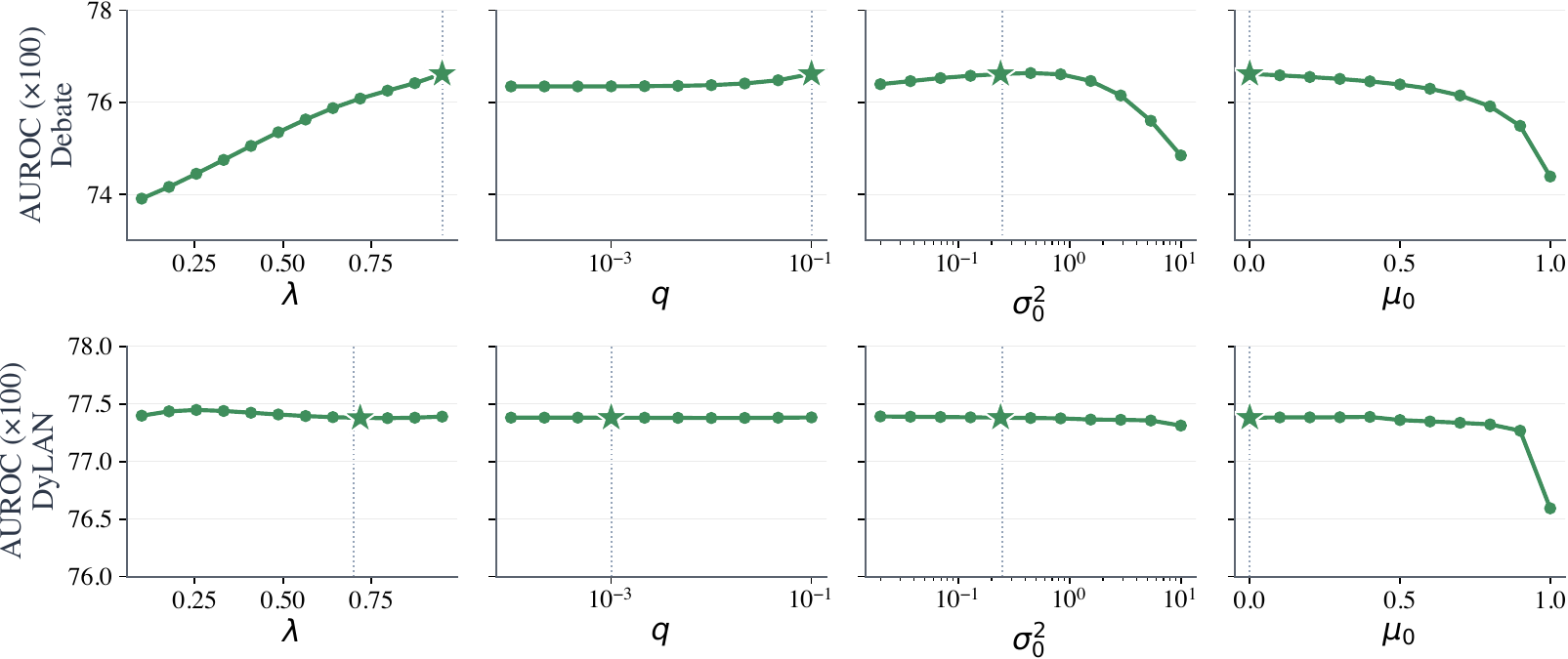}
    \caption{Per-axis sensitivity of \ours's four hyperparameters at the calibration cell (Qwen3-4B / MMLU-Pro). Top row: Debate, selected values $(\lambda, q, \sigma_0^2, \mu_0) = (0.95, 0.10, 0.25, 0.0)$. Bottom row: DyLAN, selected values $(0.70, 0.001, 0.25, 0.0)$. In each panel, one hyperparameter is swept while the other three are held at their per-protocol selected value (marked by \textcolor{spigreen}{\ding{72}} and a dotted vertical line). }
    \label{fig:ablation-hyperparams}
\end{figure}

\textbf{Setup.}\quad
\ours's four scalar hyperparameters $(\lambda, q, \sigma_0^2, \mu_0)$ are tuned by joint grid search on the Qwen3-4B / MMLU-Pro calibration cell, separately for each MAS protocol. The resulting tuples, which are $(0.95, 0.10, 0.25, 0.0)$ for Debate and $(0.70, 0.001, 0.25, 0.0)$ for DyLAN, are then transferred verbatim to every other (model, dataset) cell in Tables~\ref{tab:main-open-spi-ue}--\ref{tab:main-closed-spi-ue}.

\Figref{fig:ablation-hyperparams} sweeps each axis independently around the selected tuple. In Debate, AUROC drops by roughly three points as $\lambda$ decreases from $0.95$ to $0.1$. Decreasing $\lambda$ erodes the temporal coupling that aggregates evidence across rounds, consistent with the role of sequential dynamics highlighted in \Tabref{tab:ablation-signal}. The curves also show sensitivity to the initial mean and variance over the full swept range; these axes are not uniformly flat.

We transfer the selected tuple to every other (model, dataset) cell instead of tuning each cell separately. The sensitivity study characterises the calibration cell; it does not bound the potential benefit of per-cell tuning on other models or datasets.

\subsection{Calibration on All Benchmarks}
\label{app:calib}

\Tabref{tab:brier-open} and \Tabref{tab:brier-closed} report the full Brier-Platt grid across all backbones, MAS protocols, and benchmarks. \ours attains the best per-backbone average on every model and is the leading or near-leading method on the large majority of cells, in line with the headline result of~\Figref{fig:calibration}.

\begin{table*}[t]
    \caption{Calibration quality on open-source models, measured by Brier score after Platt scaling with 5-fold cross-validation ($\downarrow$, $\times 100$). \textbf{Bold} denotes the lowest score per (model, protocol, dataset) cell, and per-backbone Avg.}
    \label{tab:brier-open}
    \vspace{-1em}
    \begin{center}
    \resizebox{0.8\linewidth}{!}{%
    \setlength{\tabcolsep}{6pt}
    \begin{tabular}{l cc cc cc c}
    \toprule
    \multirow{2}{*}{\textbf{Method}} & \multicolumn{2}{c}{\textbf{MATH-500}} & \multicolumn{2}{c}{\textbf{MMLU-Pro}} & \multicolumn{2}{c}{\textbf{BBH}} & \multirow{2}{*}{\textbf{Avg}} \\
    \cmidrule(lr){2-3} \cmidrule(lr){4-5} \cmidrule(lr){6-7}
     & {Debate} & {DyLAN} & {Debate} & {DyLAN} & {Debate} & {DyLAN} & \\
    \midrule
    \multicolumn{8}{c}{\textbf{Qwen3-4B}} \\
    \midrule
    PE                 & 16.92 & 17.06 & 22.44 & 21.64 & 15.23 & 16.48 & 18.30 \\
    LL                 & 17.12 & 16.97 & 22.46 & 21.65 & 15.41 & 16.55 & 18.36 \\
    UDPO               & 17.16 & 15.13 & 21.40 & 19.84 & 15.54 & 15.85 & 17.49 \\
    MATU               & 17.86 & 17.94 & 22.38 & 22.34 & 16.32 & 16.16 & 18.83 \\
    \rowcolor{oursrow}
    \ours (Ours)         & \textbf{15.98} & \textbf{15.06} & \textbf{19.19} & \textbf{18.52} & \textbf{14.77} & \textbf{15.62} & \textbf{16.52} \\
    \midrule
    \multicolumn{8}{c}{\textbf{Llama-3.1-8B-Instruct}} \\
    \midrule
    PE                 & 24.02 & 25.26 & 24.84 & 24.45 & 24.59 & 22.83 & 24.33 \\
    LL                 & 24.71 & 22.51 & 24.74 & 23.65 & 24.56 & 22.49 & 23.78 \\
    UDPO               & 20.70 & 22.38 & 22.03 & 21.67 & 21.94 & 22.51 & 21.87 \\
    MATU               & 25.02 & 20.76 & 24.94 & 23.47 & 23.68 & 22.32 & 23.36 \\
    \rowcolor{oursrow}
    \ours (Ours)         & \textbf{16.57} & \textbf{17.71} & \textbf{19.76} & \textbf{20.18} & \textbf{20.67} & \textbf{20.73} & \textbf{19.27} \\
    \midrule
    \multicolumn{8}{c}{\textbf{Gemma-4-E4B-it}} \\
    \midrule
    PE                 & 13.43 & 12.79 & 18.12 & 16.79 & 11.68 & 17.23 & 15.01 \\
    LL                 & 13.41 & \textbf{12.78} & 18.14 & 16.78 & 11.69 & 17.21 & 15.00 \\
    UDPO               & 13.45 & 14.10 & 17.72 & 16.20 & 10.91 & 13.05 & 14.24 \\
    MATU               & 13.52 & 13.79 & 19.79 & 18.77 & 11.71 & 14.95 & 15.42 \\
    \rowcolor{oursrow}
    \ours (Ours)         & \textbf{13.30} & 13.79 & \textbf{15.53} & \textbf{15.29} & \textbf{10.06} & \textbf{12.40} & \textbf{13.40} \\
    \bottomrule
    \end{tabular}
    }
    \vspace{-0.5em}
    \end{center}
\end{table*}

\begin{table*}[t]
    \caption{Calibration quality on closed-source models, measured by Brier score after Platt scaling with 5-fold cross-validation ($\downarrow$, $\times 100$). Notation follows~\Tabref{tab:brier-open}.}
    \label{tab:brier-closed}
    \vspace{-1em}
    \begin{center}
    \resizebox{0.8\linewidth}{!}{%
    \setlength{\tabcolsep}{6pt}
    \begin{tabular}{l cc cc cc c}
    \toprule
    \multirow{2}{*}{\textbf{Method}} & \multicolumn{2}{c}{\textbf{OlympiadBench}} & \multicolumn{2}{c}{\textbf{MMLU-Pro}} & \multicolumn{2}{c}{\textbf{BBEH-mini}} & \multirow{2}{*}{\textbf{Avg}} \\
    \cmidrule(lr){2-3} \cmidrule(lr){4-5} \cmidrule(lr){6-7}
     & {Debate} & {DyLAN} & {Debate} & {DyLAN} & {Debate} & {DyLAN} & \\
    \midrule
    \multicolumn{8}{c}{\textbf{Gemini-2.5-Flash-Lite}} \\
    \midrule
    PE                 & 22.35 & 21.87 & 18.70 & 14.84 & 17.21 & 15.05 & 18.34 \\
    LL                 & 22.47 & 22.18 & \textbf{14.83} & \textbf{13.53} & 17.18 & 15.56 & 17.63 \\
    UDPO               & 21.60 & \textbf{16.63} & 18.47 & 15.93 & 16.80 & \textbf{14.21} & 17.27 \\
    MATU               & 22.02 & 23.45 & 18.62 & 15.93 & 17.06 & 14.56 & 18.61 \\
    \rowcolor{oursrow}
    \ours (Ours)         & \textbf{17.10} & 16.71 & 15.28 & 15.36 & \textbf{16.30} & 14.55 & \textbf{15.88} \\
    \midrule
    \multicolumn{8}{c}{\textbf{GPT-5.4-mini}} \\
    \midrule
    PE                 & 24.87 & 20.50 & 16.41 & 16.12 & 15.67 & \textbf{13.80} & 17.90 \\
    LL                 & 24.86 & 20.67 & 16.48 & 16.03 & 15.59 & 13.83 & 17.91 \\
    UDPO               & 20.92 & 17.92 & 15.69 & 14.57 & \textbf{15.30} & 13.97 & 16.40 \\
    MATU               & 25.06 & 22.96 & 16.48 & 16.47 & 15.63 & 13.83 & 18.41 \\
    \rowcolor{oursrow}
    \ours (Ours)         & \textbf{16.38} & \textbf{17.04} & \textbf{14.58} & \textbf{14.18} & 15.54 & 13.98 & \textbf{15.28} \\
    \bottomrule
    \end{tabular}
    }
    \vspace{-0.5em}
    \end{center}
\end{table*}

\subsection{Comparison with Simpler Trajectory Summary Methods}
\label{app:trajectory-summaries}

In \Secref{sec:ablation-signal}, we compared the full update with static agreement and reliability-weighted aggregation. Here, we extend the analysis to simpler trajectory summaries on Qwen3-4B and GPT-5.4-mini across six Debate settings.

The experiments in this section follow the same protocol-specific hyperparameter settings as the main experiments, as specified in \appref{app:hp-values}.

\textbf{Evaluation Records.}\quad
For the single-rollout comparisons, MMLU-Pro uses $12{,}032$ Qwen examples and $4{,}999$ evaluable GPT examples from the $5{,}000$-question subset; the remaining evaluation sizes are listed in \Tabref{tab:rebuttal-trajectory-summaries}.

\textbf{Alternative Summaries.}\quad
All scores are oriented so that higher values indicate greater uncertainty:
\begin{itemize}[nosep]
    \item \textbf{EWMA}: a fixed-weight moving average of agreement, initialized with $h_1=M_1$ and updated as $h_t=(1-\beta)h_{t-1}+\beta M_t$ with $\beta=0.3$. Its score is $-T^{-1}\sum_t h_t$, matching \ours's trajectory-averaged readout.
    \item \textbf{Slope}: the negative least-squares slope of $M_t$ over rounds.
    \item \textbf{Switches}: the number of adjacent rounds in which the plurality answer changes.
    \item \textbf{TTC} (time to consensus): the first round with unanimous agreement, with trajectories that never reach unanimity assigned $T+1$.
    \item \textbf{Mean disagreement}: $T^{-1}\sum_t(1-M_t)$.
    \item \textbf{Mean answer entropy}: $T^{-1}\sum_t H_t$, where $H_t$ is the entropy of the empirical distribution of parsed answers at round $t$. This measures cross-agent answer diversity and is distinct from token-level generation entropy $R_t$.
\end{itemize}

\begin{table}[t]
    \caption{Simpler trajectory summaries under Debate (AUROC~$\uparrow$, $\times100$). \textbf{Boldface} and \underline{underlining} denote the best and second-best performance. Evaluation sizes and score definitions follow \appref{app:trajectory-summaries}.}
    \label{tab:rebuttal-trajectory-summaries}
    \centering
    \small
    \setlength{\tabcolsep}{6pt}
    \renewcommand{\arraystretch}{1.12}
    \begin{tabular}{@{}lcccccc@{}}
    \toprule
\multirow{2}{*}{\textbf{Method}} & \multicolumn{3}{c}{\textbf{Qwen3-4B}} & \multicolumn{3}{c}{\textbf{GPT-5.4-mini}} \\
\cmidrule(lr){2-4} \cmidrule(lr){5-7}
 & MATH-500 & MMLU-Pro & BBH & OlympiadBench & MMLU-Pro & BBEH-mini \\
\midrule
EWMA & \underline{66.80} & 68.43 & 69.56 & 81.63 & 66.95 & 53.46 \\
Slope & 37.79 & 38.76 & 40.75 & 32.36 & 39.56 & \textbf{57.25} \\
Switches & 61.76 & 60.59 & 61.47 & 71.08 & 57.58 & 52.43 \\
TTC & 66.47 & 68.29 & 68.29 & 79.97 & 66.31 & \underline{56.94} \\
Mean $1{-}M_t$ & 66.60 & \underline{68.56} & \underline{69.72} & 82.33 & \underline{67.08} & 55.95 \\
Mean $H_t$ & 66.78 & 68.35 & 69.33 & \underline{82.37} & 66.80 & 56.63 \\
\rowcolor{oursrow}
\ours (Ours) & \textbf{68.59} & \textbf{76.42} & \textbf{72.86} & \textbf{83.83} & \textbf{69.53} & 49.96 \\
\midrule
$n$ & 500 & 12,032 & 1,250 & 674 & 4,999 & 460 \\
    \bottomrule
    \end{tabular}
\end{table}

\textbf{Results.}\quad
\ours achieves the highest AUROC among the listed summaries in five of the six settings. It also outperforms EWMA in those five settings, indicating that its gains generally extend beyond fixed-weight temporal smoothing. The exception is GPT-5.4-mini on BBEH-mini, where \ours is near chance and the slope summary performs best. The paired-bootstrap comparison with EWMA is reported in \appref{app:statistical-significance}.

\textbf{Temporal-Order Ablation.}\quad
We keep each trajectory and its final prediction fixed, but process the paired signals $(M_t,R_t)$ in their original order, in reverse order, or after a random permutation. Shuffled AUROC is averaged over three permutation seeds. The evaluation records and filter parameters are unchanged across conditions.

\begin{table}[t]
    \caption{Temporal-order ablation under Debate (AUROC~$\uparrow$, $\times100$). Shuffled averages three seeds; $\Delta$ is original minus the better alternative. Best and second-best scores are bold and underlined.}
    \label{tab:rebuttal-temporal-order}
    \centering
    \small
    \setlength{\tabcolsep}{6pt}
    \renewcommand{\arraystretch}{1.12}
    \begin{tabular}{@{}lcccccc@{}}
    \toprule
\multirow{2}{*}{\textbf{Method}} & \multicolumn{3}{c}{\textbf{Qwen3-4B}} & \multicolumn{3}{c}{\textbf{GPT-5.4-mini}} \\
\cmidrule(lr){2-4} \cmidrule(lr){5-7}
 & MATH-500 & MMLU-Pro & BBH & OlympiadBench & MMLU-Pro & BBEH-mini \\
\midrule
\rowcolor{oursrow}
Original & \textbf{68.59} & \textbf{76.42} & 72.86 & \textbf{83.83} & \textbf{69.53} & \textbf{49.96} \\
Reversed & 68.30 & 71.68 & \textbf{74.26} & 82.29 & 65.21 & \underline{49.72} \\
Shuffled & \underline{68.41} & \underline{72.45} & \underline{73.84} & \underline{82.99} & \underline{66.22} & 49.12 \\
\midrule
$\Delta$ & $+0.19$ & $+3.97$ & $-1.41$ & $+0.84$ & $+3.31$ & $+0.24$ \\
    \bottomrule
    \end{tabular}
\end{table}

\Tabref{tab:rebuttal-temporal-order} shows that the original order outperforms both alternatives in five of the six settings. The clearest differences occur on MMLU-Pro: original processing exceeds the better alternative by $3.97$ AUROC points for Qwen3-4B and $3.31$ for GPT-5.4-mini. The effect is smaller elsewhere, and reversing the order improves Qwen3-4B on BBH. Thus, chronological order provides useful information in several settings, but its benefit is not uniform across tasks.

\subsection{Statistical Significance}
\label{app:statistical-significance}

The experiments in this section follow the same protocol-specific hyperparameter settings as the main experiments, as specified in \appref{app:hp-values}.

\textbf{Prompt-Level Bootstrap.}\quad
We perform $1{,}000$ paired bootstrap resamples of the evaluation records in each setting of \Tabref{tab:rebuttal-trajectory-summaries}. Each resample selects records with replacement and applies the same indices to all estimators. The $2.5$th and $97.5$th percentiles give the reported $95\%$ AUROC intervals. For each baseline $B$, we also report the empirical tail fraction $p_B$, the proportion of paired resamples in which $\mathrm{AUROC}(\text{\ours})-\mathrm{AUROC}(B)\leq0$. These comparisons are one-sided and unadjusted for multiple comparisons. A reported value below $0.001$ means no nonpositive difference was observed in $1{,}000$ resamples; it does not establish a zero population probability.

\begin{table}[t]
    \caption{Paired-bootstrap comparisons under Debate. The first two rows report \ours's AUROC ($\times100$) and 95\% interval; subsequent rows give its tail fraction against each baseline. \textbf{Bold} indicates $p<0.05$; $<0.001$ denotes zero tail counts in 1,000 resamples.}
    \label{tab:rebuttal-bootstrap}
    \centering
    \small
    \setlength{\tabcolsep}{6pt}
    \renewcommand{\arraystretch}{1.12}
    \begin{tabular}{@{}lcccccc@{}}
    \toprule
\multirow{2}{*}{\textbf{Statistic}} & \multicolumn{3}{c}{\textbf{Qwen3-4B}} & \multicolumn{3}{c}{\textbf{GPT-5.4-mini}} \\
\cmidrule(lr){2-4} \cmidrule(lr){5-7}
 & MATH-500 & MMLU-Pro & BBH & OlympiadBench & MMLU-Pro & BBEH-mini \\
\midrule
\rowcolor{oursrow}
\ours & 68.59 & 76.42 & 72.86 & 83.83 & 69.53 & 49.96 \\
95\% CI & 62.6--74.7 & 75.6--77.3 & 69.4--76.3 & 80.7--87.1 & 67.4--71.5 & 43.8--56.3 \\
\midrule
vs. PE & \textbf{0.011} & $\mathbf{<0.001}$ & \textbf{0.019} & $\mathbf{<0.001}$ & $\mathbf{<0.001}$ & 0.121 \\
vs. LL & \textbf{0.020} & $\mathbf{<0.001}$ & \textbf{0.034} & $\mathbf{<0.001}$ & $\mathbf{<0.001}$ & \textbf{0.002} \\
vs. UDPO & $\mathbf{<0.001}$ & $\mathbf{<0.001}$ & $\mathbf{<0.001}$ & $\mathbf{<0.001}$ & $\mathbf{<0.001}$ & 1.000 \\
vs. MATU & \textbf{0.016} & $\mathbf{<0.001}$ & $\mathbf{<0.001}$ & $\mathbf{<0.001}$ & $\mathbf{<0.001}$ & 0.378 \\
vs. EWMA & 0.156 & $\mathbf{<0.001}$ & $\mathbf{<0.001}$ & \textbf{0.002} & $\mathbf{<0.001}$ & 0.998 \\
    \bottomrule
    \end{tabular}
\end{table}

In these comparisons, LL averages token log-likelihood over agents and rounds, and MATU uses a single rollout ($N=1$).

In these supplementary comparisons, the tail fraction is below $0.05$ against PE, averaged LL, UDPO, and single-run MATU in five of the six settings. Against EWMA, it is below $0.05$ in four settings. Qwen3-4B / MATH-500 has a positive point-estimate difference from EWMA but a tail fraction of $0.156$, while GPT-5.4-mini / BBEH-mini does not show an advantage over EWMA. These intervals quantify prompt-sampling uncertainty for a fixed rollout; they do not measure variability across independently generated rollouts.

\textbf{Repeated-Rollout Variability.}\quad
We separately summarize independently generated Debate rollouts at temperature $0.7$ on four settings, evaluating \ours, averaged LL, and final-round PE on each run. \Tabref{tab:rebuttal-rollout-variability} reports the arithmetic mean and sample standard deviation of AUROC, together with the actual run counts and evaluation sizes. MATH-500, BBH, and BBEH-mini each have five runs; OlympiadBench has six, with $673$--$674$ evaluable records per run. We restrict this summary to settings with comparable evaluation sizes across runs.

\begin{table}[t]
    \caption{Repeated Debate rollouts (AUROC~$\uparrow$, mean $\pm$ sample standard deviation, $\times100$). \textbf{Boldface} and \underline{underlining} denote the best and second-best means.}
    \label{tab:rebuttal-rollout-variability}
    \centering
    \small
    \setlength{\tabcolsep}{6pt}
    \renewcommand{\arraystretch}{1.12}
    \begin{tabular}{@{}lcccc@{}}
    \toprule
\multirow{2}{*}{\textbf{Method}} & \multicolumn{2}{c}{\textbf{Qwen3-4B}} & \multicolumn{2}{c}{\textbf{GPT-5.4-mini}} \\
\cmidrule(lr){2-3} \cmidrule(lr){4-5}
 & MATH-500 & BBH & OlympiadBench & BBEH-mini \\
\midrule
PE & 67.04 $\pm$ 1.64 & 69.83 $\pm$ 0.99 & 53.65 $\pm$ 1.92 & \underline{46.11} $\pm$ 1.31 \\
LL & \underline{67.67} $\pm$ 1.22 & \underline{71.57} $\pm$ 1.59 & \underline{70.17} $\pm$ 0.68 & 43.34 $\pm$ 2.54 \\
\rowcolor{oursrow}
\ours (Ours) & \textbf{72.26} $\pm$ 2.29 & \textbf{72.58} $\pm$ 1.76 & \textbf{82.79} $\pm$ 0.87 & \textbf{49.01} $\pm$ 1.51 \\
\midrule
Runs & 5 & 5 & 6 & 5 \\
$n$ per run & 500 & 1,250 & 673--674 & 460 \\
    \bottomrule
    \end{tabular}
\end{table}

\ours has the highest mean AUROC among the three estimators in all four settings, with sample standard deviations ranging from $0.87$ to $2.29$ AUROC points. On BBEH-mini its mean remains near chance, so the relative advantage over LL and PE should not be interpreted as reliable discrimination on that task. The repeated-run summaries and prompt-level bootstrap address different sources of variability and are reported separately.

\begin{samepage}
\subsection{Trajectory Analysis under Unanimous Final Agreement}
\label{app:path-dynamics}

The experiment in this section follows the same protocol-specific hyperparameter settings as the main experiments, as specified in \appref{app:hp-values}.

The motivating example in \Figref{fig:main} suggests that identical final agreement can conceal different interaction histories. We examine this directly on Qwen3-4B / Debate / MMLU-Pro by restricting attention to the $10{,}696$ of $12{,}032$ trajectories whose agents give the same parsed answer in the final round. We group these trajectories by the total number of individual agents' answer changes between adjacent rounds, rather than the plurality-answer switches used as a baseline in \appref{app:trajectory-summaries}.

\end{samepage}

\begin{table}[t]
    \caption{Answer-switch analysis among final-unanimous trajectories on Qwen3-4B / Debate / MMLU-Pro ($n=10{,}696$). Switches count changes in individual agents' answers.}
    \label{tab:rebuttal-path-dynamics}
    \centering
    \small
    \setlength{\tabcolsep}{6pt}
    \renewcommand{\arraystretch}{1.12}
    \begin{tabular}{@{}rccc@{}}
    \toprule
\textbf{Switches} & \textbf{$n$} & \textbf{Accuracy (\%)} & \textbf{Mean confidence} \\
\midrule
0 & 8,004 & 76.9 & 0.915 \\
1 & 1,595 & 49.5 & 0.758 \\
2 & 486 & 44.4 & 0.676 \\
3+ & 611 & 38.8 & 0.622 \\
    \bottomrule
    \end{tabular}
\end{table}

As the number of agent-level switches increases from zero to three or more, accuracy decreases from $76.9\%$ to $38.8\%$, while mean \ours confidence $\bar\mu=-u$ decreases from $0.915$ to $0.622$. Final-round agreement is constant across these groups, but \ours's score tracks their different empirical reliability. This is a descriptive association between interaction history, confidence, and correctness; the analysis does not establish a causal effect of switching, and $\bar\mu$ is not a calibrated correctness probability.

\section{Theoretical Analysis and Proofs}
\label{app:theory}

This appendix collects the formal derivations and proofs supporting the linear-Gaussian instantiation of \ours in \Secref{sec:theory}. \appref{app:proof-prop-update} proves \Propref{prop:update-gaussian} (the closed-form posterior update that defines the \ours recursion). \appref{app:proof-innovations} derives the innovations form of the marginal likelihood used as our training objective in \appref{app:constant-lambda}. \appref{app:elbo} shows that, under our linear-Gaussian model, this objective is also the tight evidence lower bound, which underwrites the variational interpretation of the empirical analysis.

\subsection{Proof of \texorpdfstring{\Propref{prop:update-gaussian}}{Proposition 4.1}}
\label{app:proof-prop-update}

\begin{proof}
We proceed by induction on $t$. The base case at $t=0$ holds by the prior $s_0 \sim \mathcal{N}(\mu_0, \sigma_0^2)$. Assume the inductive hypothesis $s_{t-1} \mid o_{1:t-1} \sim \mathcal{N}(\mu_{t-1}, \sigma_{t-1}^2)$; we show that the same Gaussian form propagates to round $t$ with the stated mean and variance.

\paragraph{Step 1: Predictive Distribution.}
The transition $s_t = \lambda s_{t-1} + \epsilon_t$ with $\epsilon_t \sim \mathcal{N}(0, q)$ independent of $(s_{t-1}, o_{1:t-1})$ is an affine transformation of a Gaussian plus an independent Gaussian, so $s_t \mid o_{1:t-1}$ is Gaussian with
\begin{align*}
\mathbb{E}[s_t \mid o_{1:t-1}]
&\;=\; \mathbb{E}[\lambda s_{t-1} + \epsilon_t \mid o_{1:t-1}]
\;=\; \lambda \mu_{t-1}
\;=:\; \tilde{\mu}_t, \\[2pt]
\mathrm{Var}[s_t \mid o_{1:t-1}]
&\;=\; \lambda^2 \,\mathrm{Var}[s_{t-1} \mid o_{1:t-1}] + \mathrm{Var}[\epsilon_t]
\;=\; \lambda^2 \sigma_{t-1}^2 + q
\;=:\; \tilde{\sigma}_t^2.
\end{align*}
Hence $s_t \mid o_{1:t-1} \sim \mathcal{N}(\tilde{\mu}_t, \tilde{\sigma}_t^2)$.

\paragraph{Step 2: Joint Distribution of State and Observation.}
The observation model $o_t \mid s_t \sim \mathcal{N}(s_t, r_t^2)$ is equivalent to $o_t = s_t + \eta_t$ with $\eta_t \sim \mathcal{N}(0, r_t^2)$ independent of $s_t$ given $o_{1:t-1}$. Conditional on $o_{1:t-1}$, both $s_t$ and $o_t$ are linear functions of jointly Gaussian variables, so the conditional joint $(s_t, o_t) \mid o_{1:t-1}$ is itself Gaussian. Its first two moments are
\begin{align*}
\mathbb{E}[o_t \mid o_{1:t-1}]
&\;=\; \mathbb{E}[s_t + \eta_t \mid o_{1:t-1}]
\;=\; \tilde{\mu}_t, \\[2pt]
\mathrm{Var}[o_t \mid o_{1:t-1}]
&\;=\; \mathrm{Var}[s_t \mid o_{1:t-1}] + \mathrm{Var}[\eta_t]
\;=\; \tilde{\sigma}_t^2 + r_t^2, \\[2pt]
\mathrm{Cov}(s_t,\, o_t \mid o_{1:t-1})
&\;=\; \mathrm{Cov}(s_t,\, s_t + \eta_t \mid o_{1:t-1})
\;=\; \mathrm{Var}[s_t \mid o_{1:t-1}]
\;=\; \tilde{\sigma}_t^2,
\end{align*}
where the last line uses the conditional independence $\mathrm{Cov}(s_t, \eta_t \mid o_{1:t-1}) = 0$. Collecting these moments,
\begin{align}\label{eq:joint-pred}
\begin{pmatrix} s_t \\ o_t \end{pmatrix} \,\bigg|\, o_{1:t-1}
\;\sim\; \mathcal{N}\!\left(\,
\begin{pmatrix} \tilde{\mu}_t \\ \tilde{\mu}_t \end{pmatrix},\;
\begin{pmatrix} \tilde{\sigma}_t^2 & \tilde{\sigma}_t^2 \\ \tilde{\sigma}_t^2 & \tilde{\sigma}_t^2 + r_t^2 \end{pmatrix}
\right).
\end{align}

\paragraph{Step 3: Conditioning on the New Observation.}
For any jointly Gaussian pair $(X, Y)$ with means $(\mu_X, \mu_Y)$, variances $(\Sigma_{XX}, \Sigma_{YY})$, and cross-covariance $\Sigma_{XY}$, the standard conditioning identity gives
\begin{align}\label{eq:gauss-cond}
X \mid Y
\;\sim\; \mathcal{N}\!\Bigl(
\mu_X + \Sigma_{XY}\Sigma_{YY}^{-1}(Y - \mu_Y),\;\;
\Sigma_{XX} - \Sigma_{XY}\Sigma_{YY}^{-1}\Sigma_{YX}
\Bigr).
\end{align}
Applying \Eqref{eq:gauss-cond} to the joint in~\Eqref{eq:joint-pred} with $X = s_t$ and $Y = o_t$ yields
\begin{align*}
s_t \mid o_{1:t}
\;\sim\; \mathcal{N}\!\left(\,
\tilde{\mu}_t + \frac{\tilde{\sigma}_t^2}{\tilde{\sigma}_t^2 + r_t^2}\bigl(o_t - \tilde{\mu}_t\bigr),
\;\;
\tilde{\sigma}_t^2 - \frac{(\tilde{\sigma}_t^2)^2}{\tilde{\sigma}_t^2 + r_t^2}\,\right).
\end{align*}
Define the Kalman gain
\begin{align*}
\alpha_t \;:=\; \frac{\tilde{\sigma}_t^2}{\tilde{\sigma}_t^2 + r_t^2} \;\in\; [0, 1].
\end{align*}
Substituting back, the posterior mean and variance simplify to
\begin{align*}
\mu_t \;=\; \tilde{\mu}_t + \alpha_t\bigl(o_t - \tilde{\mu}_t\bigr),
\qquad
\sigma_t^2 \;=\; \tilde{\sigma}_t^2 \cdot \frac{r_t^2}{\tilde{\sigma}_t^2 + r_t^2}
\;=\; (1-\alpha_t)\,\tilde{\sigma}_t^2.
\end{align*}
The variance further admits a harmonic-of-precisions form: starting from $\sigma_t^2 = \tilde{\sigma}_t^2 r_t^2 / (\tilde{\sigma}_t^2 + r_t^2)$,
\begin{align*}
\frac{1}{\sigma_t^2} \;=\; \frac{\tilde{\sigma}_t^2 + r_t^2}{\tilde{\sigma}_t^2\, r_t^2}
\;=\; \frac{1}{\tilde{\sigma}_t^2} + \frac{1}{r_t^2},
\end{align*}
which is exactly the precision-addition statement in the proposition. Since $s_t \mid o_{1:t}$ is Gaussian with the stated mean and variance, the inductive hypothesis carries to round $t$, completing the induction.
\end{proof}

\subsection{Innovations Form of the Marginal Likelihood}
\label{app:proof-innovations}

By the chain rule of probability, the log marginal likelihood factorises as $\log p(o_{1:T}) = \sum_{t=1}^T \log p(o_t \mid o_{1:t-1})$. Under \Propref{prop:update-gaussian} the posterior from the previous round is $p(s_{t-1} \mid o_{1:t-1}) = \mathcal{N}(\mu_{t-1}, \sigma_{t-1}^2)$; marginalising over the linear-Gaussian transition gives the predictive distribution
\begin{align}
p(s_t \mid o_{1:t-1}) = \mathcal{N}\!\big(\hat{\mu}_t,\; \tilde{\sigma}_t^2\big),
\quad
\hat{\mu}_t = \lambda_t(\vh_t)\,\mu_{t-1},
\quad
\tilde{\sigma}_t^2 = \lambda_t(\vh_t)^2\,\sigma_{t-1}^2 + q,
\end{align}
by the convolution of two Gaussians. Marginalising once more over $s_t$ under the linear-Gaussian observation model yields
\begin{align}
p(o_t \mid o_{1:t-1}) = \mathcal{N}\!\big(\hat{\mu}_t,\; S_t\big),
\quad
S_t = \tilde{\sigma}_t^2 + R_t,
\end{align}
and substituting the Gaussian log-density and summing over $t$ produces the innovations form
\begin{align}
\log p(o_{1:T}) = -\frac{T}{2}\log(2\pi) - \frac{1}{2}\sum_{t=1}^T \left[\log S_t + \frac{e_t^2}{S_t}\right],
\quad
e_t = o_t - \hat{\mu}_t.
\end{align}
The two terms balance prediction error against predicted uncertainty: $e_t^2/S_t$ penalises errors relative to the predictive variance, while $\log S_t$ penalises overconfidence.

\subsection{Equivalence with the Evidence Lower Bound}
\label{app:elbo}

For the following decomposition, we assume a Markov-structured variational posterior,
\[
q(s_{1:T}) = q(s_1)\prod_{t=2}^{T}q(s_t \mid s_{t-1}),
\]
where dependence on the observed sequence $o_{1:T}$ is left implicit. For any variational posterior $q(s_{1:T})$ over the latent trajectory, the standard decomposition gives
\begin{align}
\log p(o_{1:T}) &= \mathcal{L}(q) + \mathrm{KL}\!\big(q(s_{1:T}) \,\|\, p(s_{1:T} \mid o_{1:T})\big),
\quad \\
\mathcal{L}(q) &= \mathbb{E}_{q}\!\left[\log p(o_{1:T}, s_{1:T}) - \log q(s_{1:T})\right].
\end{align}
Using the Markov structure $p(o_{1:T}, s_{1:T}) = p(s_1)\prod_{t=2}^T p(s_t \mid s_{t-1})\prod_{t=1}^T p(o_t \mid s_t)$, the ELBO splits into three interpretable terms,
\begin{align}
\mathcal{L}(q) &\;=\; \underbrace{\sum_{t=1}^T \mathbb{E}_{q(s_t)}\!\left[\log p(o_t \mid s_t)\right]}_{\text{(I) observation fit}}
\;-\; \underbrace{\mathrm{KL}\!\big(q(s_1) \,\|\, p(s_1)\big)}_{\text{(II) prior regularisation}} \notag \\
&\quad \;-\; \underbrace{\sum_{t=2}^T \mathbb{E}_{q(s_{t-1})}\!\left[\mathrm{KL}\!\big(q(s_t \mid s_{t-1}) \,\|\, p(s_t \mid s_{t-1})\big)\right]}_{\text{(III) transition regularisation}},
\end{align}
each of which has a closed form under Gaussian $q$ (omitted for brevity).

Tightness then follows from a standard property of linear-Gaussian state-space models: the Kalman smoother yields the exact posterior $p(s_{1:T} \mid o_{1:T})$, so setting $q(s_{1:T}) = p(s_{1:T} \mid o_{1:T})$ makes $\mathrm{KL}(q\|p) = 0$ and the ELBO coincides with the exact log marginal likelihood. Evaluating the three terms above at the Kalman posterior and simplifying recovers the innovations form of \appref{app:proof-innovations}. The training objective in \appref{app:constant-lambda} is therefore simultaneously the exact marginal likelihood and the tight ELBO, a property we use implicitly when arguing that the failure of $\lambda_\theta(\vh_t)$ to outperform a constant cannot be attributed to a loose variational bound.


\newpage

\end{document}